\documentclass{article}
\usepackage{titletoc}
    \PassOptionsToPackage{numbers, compress}{natbib}

\usepackage[main,final]{neurips_2026}

\usepackage[utf8]{inputenc} 
\usepackage[T1]{fontenc}    
\usepackage{hyperref}       
\usepackage{url}            
\usepackage{booktabs}       
\usepackage{amsfonts}       
\usepackage{nicefrac}       
\usepackage{microtype}      
\usepackage{xcolor}         

\usepackage{enumitem}

\usepackage{graphicx}
\usepackage{subcaption}
\usepackage[table]{xcolor} 
\colorlet{ours}{blue!10}
\usepackage{wrapfig}

\usepackage{comment}
\usepackage{verbatim} 
\usepackage{amsmath}
\usepackage{amssymb}
\usepackage{mathtools}
\usepackage{amsthm}
\usepackage{algorithm}
\usepackage{algorithmic}
\usepackage{bm}
\usepackage{graphicx,booktabs,multirow}
\usepackage{subcaption}

\theoremstyle{plain}
\newtheorem{theorem}{Theorem}[section]

\newtheorem{lemma}[theorem]{Lemma}

\theoremstyle{definition}
\newtheorem{definition}[theorem]{Definition}
\newtheorem{assumption}[theorem]{Assumption}
\theoremstyle{remark}

\newcommand {\com}[1]{\tiny$\pm$#1}

\title{Communication-Efficient Personalized Adaptation via Federated-Local Model Merging}

\author{
  Yinan Zou\\
  Purdue University \\
  \texttt{zou211@purdue.edu}
  \And
  Md Kamran Chowdhury Shisher \\
  University of Colorado Colorado Springs \\
  \texttt{mshisher@uccs.edu} \\
  \AND
  Christopher G. Brinton \\
  Purdue University \\
  \texttt{cgb@purdue.edu} \\
  \And
  Vishrant Tripathi \\
  Purdue University \\
  \texttt{tripathv@purdue.edu} \\
}

\begin{document}

\maketitle

\begin{abstract}
Parameter-efficient fine-tuning methods, such as LoRA, offer a practical way to adapt large vision and language models to client tasks. However, this becomes particularly challenging under task-level heterogeneity in federated deployments. 
In this regime, personalization requires balancing \emph{general knowledge} with \emph{personalized knowledge}, yet existing approaches largely rely on heuristic mixing rules and lack theoretical justification. 
In this work, we propose \textsc{Potara}, a principled framework for federated personalization that constructs a personalized model for each client by merging two complementary models: (i) a federated model capturing \emph{general knowledge}, and (ii) a local model capturing \emph{personalized knowledge}. Through the construct of linear mode connectivity, we show that the expected task loss admits a variance trace upper bound, whose minimization yields closed-form optimal mixing weights that guarantee a tighter bound for the merged model than for either the federated or local model alone.
Experiments on vision and language benchmarks show that \textsc{Potara} consistently improves personalization while reducing communication, leading to a strong performance-communication trade-off.
\end{abstract}

\section{Introduction} \label{section:intro}

Large vision models (LVMs) and large language models (LLMs) have achieved remarkable progress, enabling their widespread adoption across diverse downstream applications \citep{zhao2023survey,awais2025foundation}.
Adapting pre-trained models to domain specific tasks typically requires fine-tuning on task relevant data, but full parameter fine-tuning is computationally expensive and impractical at scale. 
To mitigate this cost, parameter-efficient fine-tuning methods have been developed, among which low-rank adaptation (LoRA) \citep{hu2022lora} has become particularly popular.
Domain specific fine-tuning of large models often requires large heterogeneous datasets which are typically split across organizations and cannot be centralized due to privacy and security constraints.
Federated learning (FL) \citep{mcmahan2017communication} addresses this issue by enabling collaborative training while keeping raw data local.

Federated large model fine-tuning \citep{zhang2024towards} integrates FL with LoRA, allowing clients to update and communicate only LoRA modules while jointly learning a global model.
Despite its promise, federated LoRA fine-tuning faces pronounced \emph{task heterogeneity}: different clients may pursue distinct, and sometimes even conflicting, objectives.
In such settings, the knowledge required by each client naturally decomposes into two complementary components:
(i) \emph{general knowledge} that can be transferred across clients and tasks, and
(ii) \emph{personalized knowledge} that is specific to each client’s local data and task.

\textbf{Challenges.} A central challenge is how to balance cross-client \emph{general knowledge} and client-specific \emph{personalized knowledge}. 
It remains unclear, from a principled theoretical perspective, how much emphasis should be placed on general versus personalized knowledge when forming the final client model.
Although recent methods attempt to provide a solution \citep{yang2024dual,bian2025fedalt,seoco,xie2024perada}, they lack theoretical justification, leaving a gap for theoretically-grounded solutions.
Additionally, communication is the primary bottleneck in FL, and LoRA-based training still requires the repeated transmission of parameters across many rounds. 
These considerations call for an approach that (i) explicitly control the general-personalized trade-off in a principled manner and (ii) achieve favorable performance–communication trade-off in realistic federated deployments.

\subsection{Contributions}

Motivated by these limitations, this work aims to develop a framework for efficient federated fine-tuning under task heterogeneity.
Our key contributions are as follows:
\begin{itemize}[topsep=0pt,itemsep=4pt,parsep=0pt, partopsep=0pt, leftmargin=*]
    \item 
    We propose \textsc{Potara}\footnote{\textsc{Potara} derives its name from the Potara fusion in Dragon Ball, where two independent entities are merged into a single, and more powerful form.}, a federated LoRA personalization framework to jointly (i) achieve an effective trade-off between federation and personalization, and (ii) balance communication efficiency with performance. 
    \textsc{Potara} constructs a personalized model for each client by merging two complementary models: a collaborative federated model that captures \textit{general knowledge} across all clients, and a local fine-tuning model that emphasizes \textit{personalized knowledge}.
    \item 
    To determine the mixing weight between the federated model and the local model, we develop rigorous analysis grounded in linear mode connectivity (LMC). 
    We show that, within an LMC basin around the optimum, the expected excess task losses of the federated model, the local model, and their merged model admits an upper bound in which the key term is a variance trace (see Theorem \ref{loss_bound}).
    Minimizing this bound yields closed-form optimal mixing weights, which guarantee that the merged model achieves the tighter bound than both the federated model and the local model (see Lemma \ref{corollary}).
    \item 
    We conduct extensive experiments on vision and language benchmarks. The experimental results show that our method achieves superior performance compared to baseline methods while significantly reducing communication overhead, thereby offering a favorable performance-communication trade-off.
\end{itemize}

\subsection{Related Work}

\textbf{Model Merging.}
Model merging aims to combine multiple task-specific models into a single model without additional training.
\citet{wortsman2022model} utilized averaged parameters to construct the merged model.
To account for parameter importance, \citet{matena2022merging} proposed weighted fusion, where merging weights are derived from the Fisher information matrix (FIM)~\citep{fisher1922mathematical}.
\citet{jin2022dataless} formulated model merging as a linear regression problem and obtained a closed-form solution.
Instead of averaging weights, \citet{ilharco2023editing} introduced task vectors by subtracting the pre-trained parameters from each fine-tuned model, and then linearly combining these task vectors to form a merged model.
\citet{yadav2023ties} proposed to reduces conflicts in task vectors~\citep{ilharco2023editing} by trimming low-magnitude updates, resolving sign disagreements, and merging disjoint parameter subsets with consistent signs.
\citet{yu2024language} further explored this issue by randomly dropping a subset of fine-tuned weights and rescaling the remainder, yielding sparse task vectors that alleviate destructive interference.
Most existing methods \citep{ilharco2023editing,yadav2023ties,yu2024language,qicabs,stoicamodel} select merging weights via heuristic tuning (e.g., grid search), which is computation intensive.
In contrast, our approach derives the optimal mixing weights from theory, substantially improving efficiency by eliminating expensive per-client weight sweeps and repeated evaluations.

\textbf{Personalized Federated Learning.} Personalized federated learning aims to overcome the limitation of a single global model by adapting client models to their local data distribution \citep{tan2022towards}.
Common approaches include meta learning \citep{fallah2020personalized,chen2018federated}, partial model personalization \citep{collins2021exploiting,pillutla2022federated}, multi-task learning \citep{smith2017federated,li2021ditto}, hypernetworks \citep{shamsian2021personalized}, clustering \citep{duan2021flexible,ghosh2020efficient} {\color{black} and mixtures of global and local models \citep{hanzely2020federated,deng2020adaptive}}. While \citet{deng2020adaptive} studies the mixing weights of the global and local models theoretically, it depends on the distance between the client distribution and the average distribution, which is generally unavailable during federated training. In contrast, our method computes the optimal merging weight from the fisher information matrices of LoRA modules of the federated and local models.
A related study on personalized FL and model merging, FedMerge \citep{chen2025fedmerge}, proposed a multi-model FL framework in which the server maintains a pool of global models, and each client obtains a personalized model by merging these server-side models.
However, the server stores multiple global models and perform gradient-based updates for each model in the pool, both the storage and computation overhead on the server scale linearly with the number of global models. 
In contrast, our method follows a single-model paradigm where the server maintains only one global model and the server only needs to average the local models, making it substantially more resource efficient for large model fine-tuning.
Moreover, the above methods primarily address statistical heterogeneity (e.g., label distribution skew) and are typically evaluated on conventional small model architectures.
In contrast, we study a more realistic and challenging regime for large model adaptation, where heterogeneity manifests not only in label diversity but also in task diversity across clients.

\textbf{Federated Fine-Tuning.} FedIT~\citep{zhang2024towards} integrated LoRA with FedAvg and is one of the earliest efforts to combine parameter-efficient adaptation with federated optimization.
Subsequent works studied federated LoRA under differential privacy \citep{sun2024improving}, the asymmetric roles of LoRA matrices \citep{guo2024selective}, and heterogeneous-rank LoRA matrices \citep{fang2025federated,cho2024heterogeneous,bai2024federated,wang2024flora,zou2025joint}.
\citep{hao2025personalized,qi2024fdlora} studied personalization under label-distribution heterogeneity while still assuming task-homogeneous adaptation.
Considering task heterogeneity, \citep{yang2024dual,bian2025fedalt,seoco} introduced global and local LoRA modules and combine them via a dynamic mixer.
Although empirically effective, such designs lack theoretical justification for determining the appropriate mixing weights between these two LoRA modules.
In contrast, our method derives the optimal mixing weight from a principled theoretical analysis.

\vspace{-2mm}
\section{Preliminaries and Notation}

In this section, we introduce the preliminary concepts and notation used throughout the paper, including LoRA, federated fine-tuning, and linear mode connectivity.
\vspace{-2mm}
\subsection{Parameter-Efficient Fine-Tuning: LoRA}
Consider a pre-trained model with parameters $\mathbf{W}_{\mathrm{Pre}} \in \mathbb{R}^{m \times n}$.
We denote
$\Delta \mathbf{W} \in \mathbb{R}^{m \times n}$ as the trainable update matrix during fine-tuning.
Instead of updating all parameters in $\Delta \mathbf{W}$, LoRA \citep{hu2022lora} decomposes 
$\Delta \mathbf{W}$ into two low-rank matrices 
$\mathbf{A} \in \mathbb{R}^{r \times n}$ and $\mathbf{B} \in \mathbb{R}^{m \times r}$, 
where $r \ll \min(m, n)$. This decomposition allows the fine-tuning to focus on the 
significantly smaller low-rank matrices $\mathbf{A}$ and $\mathbf{B}$ instead of the full 
matrix $\Delta \mathbf{W}$. 
Consequently, the total number of trainable parameters is reduced 
from $m \times n$ to $r \times (m + n)$. The model parameters after fine-tuning are given by $\mathbf{W}_{\mathrm{FT}} = \mathbf{W}_{\mathrm{Pre}} + \Delta \mathbf{W} 
    = \mathbf{W}_{\mathrm{Pre}} + \mathbf{B}\mathbf{A}$,
where $\mathbf{A}$ is typically initialized with random Gaussian values and $\mathbf{B}$ is initialized to zero \citep{hu2022lora,hayou2024lora}.

\subsection{Federated LoRA Fine-Tuning}

FedIT \citep{zhang2024towards} integrates LoRA \citep{hu2022lora} with the FedAvg \citep{mcmahan2017communication} framework.
FedIT aims to minimize a global loss function subject to the consensus constraints ${\bf B}^{(1)} = \ldots = {\bf B}^{(N)}$ and ${\bf A}^{(1)} = \ldots = {\bf A}^{(N)}$ where ${\bf B}^{(i)}$ and ${\bf A}^{(i)}$ denotes the LoRA matrices at client $i$.
Consequently, it ultimately yields a single model shared across all clients.

\subsection{Linear Mode Connectivity}

Prior works \citep{draxler2018essentially,garipov2018loss} show that the minima found by two independently trained neural networks can be connected by a simple path along which the loss or accuracy remains nearly constant.
This phenomenon is referred to as mode connectivity.
Subsequent studies \citep{nagarajan2019uniform,frankle2020linear} further show that when neural networks are trained from the same initialization, the resulting well-trained models can be connected by a linear path, known as linear mode connectivity (LMC).
We restate the definition of LMC used in \citep{frankle2020linear,zhou2023going}.
\begin{definition}\citep{frankle2020linear,zhou2023going}\label{def:lmc}
Given a dataset $\mathcal{D}$ and two well-trained models $\mathbf{W}_1$ and $\mathbf{W}_2$ such that
$ f(\mathbf{W}_1;\mathcal{D}) \approx f(\mathbf{W}_2;\mathcal{D}) $,
we say that $\mathbf{W}_1$ and $\mathbf{W}_2$  are linearly connected if they satisfy
\begin{align}
    &f(\lambda_1\mathbf{W}_1+\lambda_2\mathbf{W}_2 ;\mathcal{D}) \approx
    f(\mathbf{W}_1;\mathcal{D}),
    \,\forall\,\lambda_1,\lambda_2\ge 0,\,\lambda_1 + \lambda_2=1,
\end{align}
where $f(\mathbf{W};\mathcal{D})$ denotes the loss of  model $\mathbf{W}$ on dataset $\mathcal{D}$.
\end{definition}

In Definition \ref{def:lmc}, two models $\mathbf{W}_1$ and $\mathbf{W}_2$ are linearly connected if the loss of the interpolated mode remains nearly unchanged.
In other words, $\mathbf{W}_1$ and $\mathbf{W}_2$ lie within the same flat basin of the loss landscape, and their linear interpolation remains inside this basin \citep{neyshabur2020being}.

\section{Our Method: \textsc{Potara}}
\label{section:method}

In this section, we propose \textsc{Potara}, which achieves a principled trade-off between federation and personalization while simultaneously balancing communication efficiency and performance.

\subsection{Motivation}

\begin{wrapfigure}{r}{0.575\textwidth}
    \vspace{-15mm}
    \begin{center}
    \includegraphics[width=.575\textwidth]{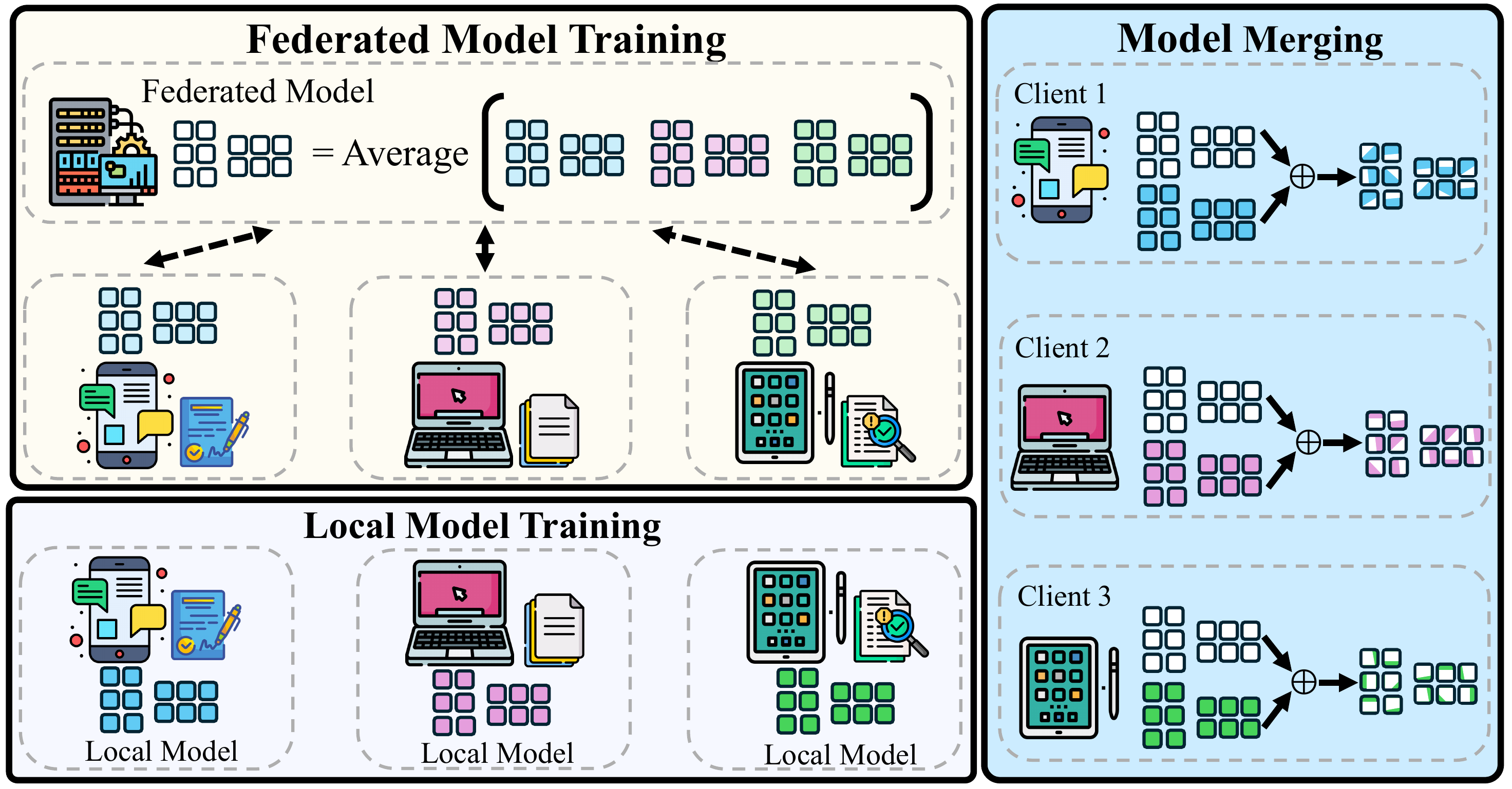}
    \end{center}
    \vspace{-3mm}
    \caption{\small Illustration of \textsc{Potara}. 
    We train a FedIT model via FL and train local fine-tuning models. 
    We obtain the final personalized model by weighted merging these FedIT and local models.}
    \vspace{-2.6mm}
\end{wrapfigure}

In large model fine-tuning, different clients usually hold data from distinct tasks (e.g., question classification versus sentiment analysis)~\citep{fallah2020personalized,collins2021exploiting,bian2025fedalt,yang2024dual,tian2023distributed}. 
In such task-heterogeneous settings, consensus-based FL methods often underperform, since enforcing a single shared solution neglects both task diversity and the need for client-specific adaptation. 
A trivial alternative is to let each client fine-tune independently without communication. 
However, the number of local data samples is typically insufficient to train a satisfactory model in practical settings. 
By exchanging information across clients, each client can exploit transferable knowledge learned elsewhere to improve its own model. 
Accordingly, personalized federated fine-tuning aims to learn a client-specific model that effectively combines \emph{general knowledge} across clients with \emph{personalized knowledge} adapted to the client’s local task.

\vspace{-3mm}
\subsection{Federated-Local Model Merging}

Model merging provides a lightweight mechanism for combining complementary capabilities from multiple fine-tuned models without additional joint training.
Building on this idea, we propose \textsc{Potara}, which constructs a personalized model for each client by merging two complementary models:
\begin{itemize}[noitemsep]
    \item A collaborative FedIT model learned through cross-client communication,
    \item A purely local model fine-tuned only on that client\textquotesingle s data.
\end{itemize}
Intuitively, the FedIT model captures \emph{general knowledge} distilled from all clients.
In contrast, the local model emphasizes \emph{personalized knowledge} tailored to the client’s task and distribution.
By merging them, \textsc{Potara} yields a client-specific model that simultaneously leverages federation and personalization, offering an effective trade-off between these two objectives.

Let $\bm{\theta}^{(i)}_{\mathrm{Local}}$ denote the task vector corresponding to the LoRA update of client $i$'s locally fine-tuned model, and let $\bm{\theta}_{\mathrm{FedIT}}$ denote the task vector corresponding to the LoRA update of the FedIT model \citep{ilharco2023editing}.
For each client $i$, we form the merged task vector by a convex combination:
\begin{align}
\bm{\theta}^{(i)}_{\mathrm{Merge}} 
&= \lambda^{(i)}_{\mathrm{FedIT}}\,\bm{\theta}_{\mathrm{FedIT}} + \lambda^{(i)}_{\mathrm{Local}}\,\bm{\theta}^{(i)}_{\mathrm{Local}},
\end{align}
where $\lambda^{(i)}_{\mathrm{FedIT}},\lambda^{(i)}_{\mathrm{Local}} \geq 0$ denote the mixing weights for client $i$ and $\lambda^{(i)}_{\mathrm{FedIT}}+\lambda^{(i)}_{\mathrm{Local}}=1$.
Then, the final model for each client is given by
\begin{align}
    \mathbf{W}_{\mathrm{Final}}^{(i)} = \mathbf{W}_{\mathrm{Pre}} + \bm{\theta}^{(i)}_{\mathrm{Merge}}.
\end{align}
A natural question then arises: how to choose the mixing weights
$\lambda^{(i)}_{\mathrm{FedIT}}$ and $\lambda^{(i)}_{\mathrm{Local}}$?
This choice is critical, as it governs the trade-off between leveraging \emph{general knowledge} from FedIT and preserving \emph{personalized knowledge} from local fine-tuning.
Next, we present our weighting strategy for determining $\lambda^{(i)}_{\mathrm{FedIT}}$ and $\lambda^{(i)}_{\mathrm{Local}}$ in a principled manner.

\subsection{Determination of Mixing Weights} \label{mixing_weight}

In the preliminaries, we recalled the standard definition of LMC, which characterizes the phenomenon in terms of the loss remaining nearly constant along the linear interpolation path.
To obtain a local geometric characterization of LMC, we introduce the notion of LMC basin below.

\begin{definition}\label{LMC-basin}
For a given radius $\delta >0$, define the LMC basin around the optimum $\mathbf{W}^*$ as
\begin{align}
    \mathcal U_\delta 
    =
    \big\{ \mathbf{W}:\ 
    \| \mathbf{W} - \mathbf{W}^* \| \le \delta \big\}.
\end{align}
Two models $\mathbf{W}_1$ and $\mathbf{W}_2$ exhibit linear mode connectivity in $\mathcal U_\delta$
if all their convex combination lies entirely in $\mathcal{U}_\delta$, i.e.,
\begin{align}
    \lambda_1 \mathbf{W}_1 + \lambda_2 \mathbf{W}_2 \in \mathcal{U}_\delta,
    \quad  \forall\,\lambda_1,\lambda_2\ge 0,\,\lambda_1 + \lambda_2=1.
\end{align}
\end{definition}

We assume that the FedIT and each local model exhibit LMC within the LMC basin, i.e., the entire linear interpolation path between them remains inside the basin.
\begin{assumption}\label{assump:basin_event}
Let $\mathcal U_\delta$ be the LMC basin in Definition~\ref{LMC-basin}. For each client $i$, we assume that
$\mathbf{W}_{\mathrm{Pre}} 
+ \lambda^{(i)}_{\mathrm{FedIT}}\,\bm{\theta}_{\mathrm{FedIT}}
+ \lambda^{(i)}_{\mathrm{Local}}\,\bm{\theta}^{(i)}_{\mathrm{Local}}
\;\in\; \mathcal U_\delta,
\quad \forall\,\lambda^{(i)}_{\mathrm{FedIT}},\lambda^{(i)}_{\mathrm{Local}}\ge 0,\ 
\lambda^{(i)}_{\mathrm{FedIT}}+\lambda^{(i)}_{\mathrm{Local}}=1.$
\end{assumption}

Note that our setting is fine-tuning from a shared pre-trained model, so both the federated and local models can be viewed as relatively small perturbations around an already strong model for related tasks. This suggests that they may remain in a common basin. We provide empirical evidence that supports the LMC assumption under task heterogeneity in Appendices~\ref{appendix:mixing_weight}.
Throughout our analysis, we focus on the behavior of the loss landscape within the LMC basin and make the following  standard assumptions on optimization.

\begin{assumption}\label{smoothness}
The task loss function $\mathcal{L}$ is locally Lipschitz smooth within the LMC basin $\mathcal U_\delta$.
There exist constants \( L > 0\) such that
\begin{align}
     \mathcal{L}(\mathbf{y}) \leq \mathcal{L}(\mathbf{x})
    + \nabla \mathcal{L}(\mathbf{x})^\top (\mathbf{y} - \mathbf{x}) + \frac{L}{2} \| \mathbf{y} - \mathbf{x} \|^2,
    \,\forall\,\mathbf{x},\mathbf{y}\in\mathcal{U}_\delta.
\end{align}
Equivalently, the Hessian of $\mathcal{L}$ is bounded  within the LMC basin, i.e., $\nabla^2 \mathcal{L}(\mathbf{x})
    \preceq
    L\,\mathbf{I},
    \, \forall\, \mathbf{x}\in \mathcal U_\delta$.
\end{assumption}

\begin{assumption}\cite{nesterov2006cubic,carmon2018accelerated}\label{assump:L_H}
The Hessian of the task loss within the LMC basin is $L_H$-Lipschitz continuous, i.e., $
        \| \nabla^2 \mathcal{L}(\mathbf{y}) -  \nabla^2 \mathcal{L}(\mathbf{x})  \| \leq L_H \| \mathbf{y} - \mathbf{x} \|, \forall\,\mathbf{x},\mathbf{y}\in\mathcal{U}_\delta$, where $ L_H>0$ is a constant.
\end{assumption}

In the following, we show that the model obtained by merging the FedIT model and the local model yields a personalized performance gain.
For clarity, we omit the client index.

We use the Laplace approximation to approximate the task vectors of FedIT (i.e., $\bm{\theta}_{\mathrm{FedIT}}$) and local fine-tuning (i.e., $\bm{\theta}_{\mathrm{Local}}$) as Gaussian distributions:
\begin{align}
    \bm{\theta}_{\mathrm{FedIT}} &\sim \mathcal{N}(\bm{\mu}_{\mathrm{FedIT}},\bm{\Sigma}_{\mathrm{FedIT}}), 
    \\
    \bm{\theta}_{\mathrm{Local}} &\sim \mathcal{N}(\bm{\mu}_{\mathrm{Local}},\bm{\Sigma}_{\mathrm{Local}}),
\end{align}
where $\bm{\mu}_{\mathrm{FedIT}}$ and $\bm{\mu}_{\mathrm{Local}}$ denote the MAP parameter  estimates of the corresponding task vectors, and
$\bm{\Sigma}_{\mathrm{FedIT}}=\mathbf{F}_{\bm{\mu}_\mathrm{FedIT}}^{-1}$ and $\bm{\Sigma}_{\mathrm{Local}}=\mathbf{F}_{\bm{\mu}_\mathrm{Local}}^{-1}$ are given by the inverses of their Fisher information matrices.
The full derivation is provided in Appendix \ref{appendix:A1}.

We further assume a joint Gaussian distribution of FedIT and local task vectors, which introduces an explicit cross-covariance term.

\begin{assumption}\label{joint_Gaussian} \cite{ritter2018scalable,daheimmodel,matena2022merging,wang2025more}
    We assume that $(\bm{\theta}_{\mathrm{FedIT}}, \bm{\theta}_{\mathrm{Local}})$ follow a joint Gaussian distribution:
\begin{align}
         \begin{bmatrix}
         \bm{\theta}_{\mathrm{FedIT}}
         \\ \bm{\theta}_{\mathrm{Local}}
        \end{bmatrix} 
        \sim \mathcal{N} \bigg( \begin{bmatrix}
         \bm{\mu}_{\mathrm{FedIT}}
         \\ \bm{\mu}_{\mathrm{Local}}
        \end{bmatrix},
        \begin{bmatrix}
            \bm{\Sigma}_{\mathrm{FedIT}} &\bm{\Sigma}_{\mathrm{Cross}}
           \\  \bm{\Sigma}_{\mathrm{Cross}}^\top & \bm{\Sigma}_{\text{Local}}
        \end{bmatrix} \bigg),
\end{align}
where $\bm{\Sigma}_{\mathrm{Cross}}$ denotes the cross-covariance matrix between the two task vectors.
We assume that FedIT task and local task are correlated, therefore, the cross-covariance matrix satisfies $\bm{\Sigma}_{\mathrm{Cross}} \succeq 0$.
Furthermore, we assume that $\bm{\Sigma}_{\mathrm{Cross}} \preceq \bm{\Sigma}_{\mathrm{FedIT}}$ and $\bm{\Sigma}_{\mathrm{Cross}} \preceq \bm{\Sigma}_{\mathrm{Local}}$.
\end{assumption}

The conditions  $\bm{\Sigma}_{\mathrm{Cross}} \preceq \bm{\Sigma}_{\mathrm{FedIT}}$ and $\bm{\Sigma}_{\mathrm{Cross}} \preceq \bm{\Sigma}_{\mathrm{Local}}$ indicate that the cross-covariance is not larger than the marginal variances, i.e., the shared uncertainty does not dominate each model's own uncertainty. {\color{black} We provide further discussion and justification of this assumption in Appendix~\ref{appendix:Gaussian_assumption}.}

With the joint Gaussian model in Assumption \ref{joint_Gaussian}, we can derive the Gaussian distribution of the merged task vector:
\begin{align}
    \bm{\theta}_{\mathrm{Merge}} \sim \mathcal{N}(\bm{\mu}_{\mathrm{Merge}}, \bm{\Sigma}_{\mathrm{Merge}}),
\end{align}
where
\begin{align}
    \bm{\mu}_{\mathrm{Merge}} &= \lambda_\mathrm{FedIT} \bm{\mu}_\mathrm{FedIT} + \lambda_\mathrm{Local} \bm{\mu}_\mathrm{Local}, \\
    \bm{\Sigma}_{\mathrm{Merge}} 
    &= \lambda_\mathrm{FedIT}^2 \bm{\Sigma}_\mathrm{FedIT} 
      + \lambda_\mathrm{Local}^2 \bm{\Sigma}_\mathrm{Local} 
       + \lambda_\mathrm{FedIT} \lambda_\mathrm{Local} (\bm{\Sigma}_{\mathrm{Cross}} + \bm{\Sigma}_{\mathrm{Cross}}^{\!\top}).
\end{align}

We now present an expected excess loss bound for FedIT, local, and their merged model.
\begin{theorem}\label{loss_bound}
Under Definition~\ref{LMC-basin}, Assumption~\ref{assump:basin_event}, \ref{smoothness},  \ref{assump:L_H}, and \ref{joint_Gaussian}, for $i\in\{\mathrm{FedIT},\mathrm{Local},\mathrm{Merge}\}$, it holds that 
\begin{align}
    &\mathbb{E}\!\left[\mathcal{L}(\mathbf{W}_{\mathrm{Pre}}+\bm{\theta}_i)-\mathcal{L}(\mathbf{W}^*)\right]
    \le
    \bigg(\frac{L}{2} + \frac{L_H}{6}\delta \bigg)\delta^2  
    + \bigg( \frac{L}{2}
    + \frac{L_H}{6}\delta \bigg) \mathrm{tr}(\bm{\Sigma}_{i}).
\end{align}
\end{theorem}

Theorem~\ref{loss_bound} indicates that the excess loss upper bound is primarily controlled by the trace term.
For the merged model, $\mathrm{tr}(\bm{\Sigma}_{\mathrm{Merge}})$ depends on the mixing weights $(\lambda_{\mathrm{FedIT}}, \lambda_{\mathrm{Local}})$.
In the following, we minimize this trace term to obtain the optimal mixing weights in closed form.

\begin{lemma}\label{corollary}
Minimizing the trace term  $\mathrm{tr}(\bm{\Sigma}_{\mathrm{Merge}})$ yields the optimal mixing weights, given by
\begin{align}\label{eq:optimal_weight}
    \lambda_{\mathrm{FedIT}}^{*} = \frac{b - c}{a + b - 2c}, 
    \qquad
    \lambda_{\mathrm{Local}}^{*} = \frac{a - c}{a + b - 2c},
\end{align}
where $
a = \mathrm{tr}(\bm{\Sigma}_{\mathrm{FedIT}}), 
\,
b = \mathrm{tr}(\bm{\Sigma}_{\mathrm{Local}}), 
\,
c = \mathrm{tr}(\bm{\Sigma}_{\mathrm{Cross}}).
$
At this optimum, the merged trace is smaller than those of the two individual models:
\begin{align}
    \mathrm{tr}(\bm{\Sigma}_{\mathrm{Merge}}) 
\leq \min\{\mathrm{tr}(\bm{\Sigma}_{\mathrm{FedIT}}),\, \mathrm{tr}(\bm{\Sigma}_{\mathrm{Local}})\}.
\end{align}
\end{lemma}

In \eqref{eq:optimal_weight}, a smaller value of $a$ or $b$ implies lower variance and hence a more reliable model, which leads to a larger weight for that model.
At the optimal weights $(\lambda^*_{\mathrm{FedIT}}, \lambda^*_{\mathrm{Local}})$ in Lemma~\ref{corollary}, $\mathrm{tr}(\bm{\Sigma}_{\mathrm{Merge}})$ attains its minimum.
Moreover, this minimum satisfies $ \mathrm{tr}(\bm{\Sigma}_{\mathrm{Merge}}) 
\leq \min\{\mathrm{tr}(\bm{\Sigma}_{\mathrm{FedIT}}),\, \mathrm{tr}(\bm{\Sigma}_{\mathrm{Local}})\}$.
Therefore, by Theorem~\ref{loss_bound} and Lemma~\ref{corollary}, the merged model admits an excess loss upper bound that is tighter than the bounds for FedIT or local fine-tuning alone.

\vspace{-2mm}
\subsection{Estimation of Fisher Information Matrix}

In Section \ref{mixing_weight}, we established that $\bm{\Sigma}_{\mathrm{FedIT}}=\mathbf{F}_{\bm{\mu}_\mathrm{FedIT}}^{-1}$ and $\bm{\Sigma}_{\mathrm{Local}}=\mathbf{F}_{\bm{\mu}_\mathrm{Local}}^{-1}$, i.e., each posterior variance is given by the inverse Fisher information matrix.
To reduce the computation cost, we adopt a diagonal approximation
of the Fisher information matrix \citep{matena2022merging,lee2025dynamic}. 
The diagonal Fisher matrix evaluated at the posterior mean $\bm{\mu}_i$ is given by
\begin{align}
\widehat{\mathbf{F}}_{\bm{\mu}_i}
&=
\text{diag}\bigg(
\frac{1}{N}\sum_{n=1}^{N}
\mathbb{E}_{\bm{y} \sim p_{\bm{\mu}_i}(\bm{y} \mid \bm{x}_n)}
\left[
\big(
\nabla_{\bm{\mu}_i} \log p_{\bm{\mu}_i}(\bm{y} \mid \bm{x}_n)
\big)^{\odot 2}
\right]\bigg),
\label{eq:diag_fisher}
\end{align}
where $N$ is the mini-batch size to estimate the Fisher information and $(\cdot)^{\odot 2}$ denotes the element-wise square operator.
Consequently, the posterior variance can be approximated as $\bm{\Sigma}_i \approx (\widehat{\mathbf{F}}_{\bm{\mu}_i})^{{-1}}$.

To approximate the cross-covariance between the two task vectors, we compute per-sample
gradients under $\bm{\mu}_{\mathrm{FedIT}}$ and $\bm{\mu}_{\mathrm{Local}}$, estimate their coordinate-wise variances and covariance,
and obtain a gradient correlation coefficient $\rho_k$ for each parameter dimension.
We then construct a diagonal approximation
{\setlength{\abovedisplayskip}{3.5pt}
 \setlength{\belowdisplayskip}{3.5pt}
\begin{align}
    (\bm{\Sigma}_{\mathrm{Cross}})_{kk} \;\approx\; \rho_k \sqrt{(\bm{\Sigma}_{\mathrm{FedIT}})_{kk}(\bm{\Sigma}_{\mathrm{Local}})_{kk}},
\end{align}}%
where $k$ denotes the parameter coordinate and $\rho_k$ is clipped to ensure $\bm{\Sigma}_{\mathrm{Cross}}\succeq 0$ and
$\bm{\Sigma}_{\mathrm{Cross}} \preceq \bm{\Sigma}_{\mathrm{FedIT}}, \bm{\Sigma}_{\mathrm{Local}}$.
Further details on the merging procedure are provided in the Appendix \ref{appendix:estimation_procedure}.

\section{Experiments}\label{section:experiments}

In this section, we evaluate and compare the performance of the proposed \textsc{Potara} against baseline methods on vision and language benchmarks.
The baseline methods
include local fine-tuning \citep{hu2022lora}, 
FedIT \citep{zhang2024towards}, FFA-LoRA \citep{sun2024improving},  FedSA \citep{guo2024selective},
APFL \citep{deng2020adaptive}, L2GD \citep{hanzely2020federated},
FedDPA \citep{yang2024dual}, FedALT \citep{bian2025fedalt}, FedMosaic \citep{seoco}, PerAda \citep{xie2024perada},
and Fisher Merging \citep{matena2022merging}.

\begin{table*}[t]
\centering
\scriptsize

\captionof{table}{Accuracy comparison on personalized CIFAR-100 benchmark. \textsc{Potara} obtains consistent improvement over the baselines across all clients.}
\label{table_cifar100}
\resizebox{\textwidth}{!}{
\begin{tabular}{l|cccccccccc|c}
\multicolumn{12}{c}{\textbf{Personalized CIFAR-100 benchmark (ViT-B/16 model)}} \\
\toprule
\textbf{Methods} & \textbf{Client 1} & \textbf{Client 2} & \textbf{Client 3} & \textbf{Client 4} & \textbf{Client 5} & \textbf{Client 6} & \textbf{Client 7} & \textbf{Client 8} & \textbf{Client 9} & \textbf{Client 10} & \textbf{Avg.} \\
\midrule
Local & 60.01 {\com 0.48} & 59.78 {\com 0.94} & 63.26 {\com 2.68} & 61.79 {\com 2.06} & 61.26 {\com 1.07} & 56.46 {\com 4.77} & 60.62 {\com 1.72} & 61.79 {\com 1.75} & 58.75 {\com 2.82} & 60.53 {\com 4.29} & 60.43 {\com 0.93} \\
FFA-LoRA & 39.96 {\com 3.87} & 44.55 {\com 2.91} & 42.36 {\com 6.53} & 40.14 {\com 2.98} & 42.85 {\com 1.62} & 38.62 {\com 0.80} & 41.50 {\com 3.63} & 40.03 {\com 2.37} & 44.07 {\com 1.86} & 41.08 {\com 6.12} & 41.52 {\com 2.44} \\
FedSA  & 62.24 {\com 0.44} & 63.11 {\com 1.41} & 61.79 {\com 5.53} & 61.16 {\com 2.72} & 59.44 {\com 0.99} & 59.17 {\com 1.94} & 61.11 {\com 1.54} & 59.73 {\com 0.45} & 59.71 {\com 2.02} & 61.24 {\com 4.27} & 60.87 {\com 0.97} \\
FedIT  & 62.55 {\com 2.33} & 60.53 {\com 2.37} & 60.30 {\com 4.86} & 59.48 {\com 3.38} & 62.07 {\com 3.63} & 59.99 {\com 1.51} & 61.93 {\com 3.14} & 62.73 {\com 1.52} & 63.60 {\com 3.15} & 59.61 {\com 3.35} & 61.28 {\com 0.29} \\
FedDPA  & 62.80 {\com 1.46} & 60.51 {\com 2.79} & 61.22 {\com 4.81} & 62.41 {\com 2.87} & 61.54 {\com 2.43} & 60.31 {\com 1.71} & 63.24 {\com 1.94} & 62.09 {\com 2.13} & 62.74 {\com 1.82} & 61.14 {\com 4.61} & 61.80 {\com 0.96} \\
Fisher Merging & 60.25 {\com 3.04} & 60.27 {\com 1.96} & 63.28 {\com 4.56} & 59.76 {\com 1.54} & 60.99 {\com 0.56} & 58.06 {\com 1.05} & 59.13 {\com 2.13} & 58.66 {\com 0.71} & 58.08 {\com 2.43} & 57.57 {\com 4.86} & 59.60 {\com 1.13} \\
L2GD  & 55.50 {\com 2.43} & 52.75 {\com 4.07} & 57.43 {\com 9.88} & 56.27 {\com 2.29} & 54.12 {\com 1.31} & 52.46 {\com 2.62} & 55.39 {\com 4.07} & 52.84 {\com 4.66} & 56.87 {\com 6.99} & 50.63 {\com 10.88} & 54.43 {\com 4.92} \\
FedMosaic  & 62.88 {\com 0.89} & 62.27 {\com 1.23} & 62.26 {\com 4.73} & 62.05 {\com 2.37} & 62.98 {\com 1.09} & 58.28 {\com 0.93} & 62.36 {\com 1.99} & 60.19 {\com 0.32} & 61.49 {\com 2.99} & 60.59 {\com 6.45} & 61.53 {\com 1.29} \\
APFL & 63.36 {\com 1.36} & 61.45 {\com 1.95} & 63.44 {\com 2.24} & 61.48 {\com 1.61} & 64.06 {\com 0.34} & 58.60 {\com 1.46} & 63.11 {\com 1.87} & 61.56 {\com 1.04} & 61.81 {\com 0.29} & 60.14 {\com 4.30} & 61.91 {\com 0.02} \\
PerAda  & 60.79 {\com 0.85} & 61.58 {\com 0.31} & 60.36 {\com 3.25} & 60.77 {\com 1.09} & 62.90 {\com 0.14} & 61.22 {\com 1.68} & 61.51 {\com 0.01} & 58.58 {\com 3.89} & 60.35 {\com 0.02} & 58.35 {\com 6.51} & 60.64 {\com 0.54} \\
\rowcolor{ours} \textsc{Potara} & 63.66 {\com 2.48} & 64.95 {\com 1.06} & 65.15 {\com 5.59} & 62.50 {\com 2.24} & 62.47 {\com 1.95} & 60.36 {\com 1.64} & 63.54 {\com 2.72} & 62.58 {\com 1.21} & 63.14 {\com 0.73} & 61.51 {\com 3.17} & 62.98 {\com 1.47} \\
\bottomrule
\end{tabular}
}

\vspace{2mm}

\captionof{table}{Accuracy comparison on commonsense reasoning benchmark. As in Table~\ref{table_cifar100}, \textsc{Potara} obtains consistent improvements.}
\label{table_cs}
\resizebox{\textwidth}{!}{
\begin{tabular}{l|cccccccc|c}
\multicolumn{10}{c}{\textbf{Commonsense reasoning benchmark (Llama-3.2-3B-Instruct model)}} \\
\toprule
\textbf{Methods} & \textbf{ARC-c} & \textbf{ARC-e} & \textbf{BoolQ} & \textbf{HellaSwag} & \textbf{OBQA} & \textbf{PIQA} & \textbf{SIQA} & \textbf{WinoGrande} & \textbf{Avg.} \\
\midrule
Local       &  72.13 {\com  0.90} & 86.20 {\com 1.66} & 66.27 {\com 2.78} & 68.93 {\com 1.54} & 75.40 {\com 0.75} & 77.27 {\com 1.36} & 66.13 {\com 1.64} & 64.07 {\com 2.37} & 72.05 {\com 0.25} \\
FFA-LoRA     & 73.27 {\com 0.50} & 85.53 {\com 0.52} & 64.80 {\com 0.91} & 57.73 {\com 2.78} & 73.87 {\com 0.50} & 79.07 {\com 1.52} & 62.00 {\com 7.09} & 54.13 {\com 1.18} & 68.80 {\com 0.77} \\
FedSA         & 71.53 {\com 1.25} & 84.67 {\com 1.05} & 64.27 {\com 2.36} & 64.93 {\com 2.73} & 74.87 {\com 0.93} & 77.73 {\com 0.81} & 66.53 {\com 1.39} & 64.67 {\com 2.84} & 71.15 {\com 0.62} \\
FedIT       & 74.27 {\com 0.62} & 85.93 {\com 1.46} & 66.87 {\com 1.18} & 64.93 {\com 1.05} & 74.33 {\com 1.64} & 78.80 {\com 1.77} & 69.87 {\com 1.09} & 60.53 {\com 1.75} & 71.94 {\com 0.42} \\
FedALT & 72.73 {\com 0.46} & 85.00 {\com 0.35} & 64.27 {\com 0.58} & 68.33 {\com 0.92} & 76.67 {\com 0.12} & 79.00 {\com 0.00} & 65.53 {\com 0.58} & 60.33 {\com 4.39} & 71.48 {\com 0.58} \\
FedDPA        & 74.73 {\com 1.25} & 87.60 {\com 0.99} & 63.80 {\com 1.23} & 65.73 {\com 1.39} & 77.07 {\com 0.50} & 80.40 {\com 1.88} & 67.87 {\com 0.66} & 61.47 {\com 0.66} & 72.33 {\com 0.37} \\
L2GD  & 74.70 {\com 0.30} & 88.30 {\com 1.10} & 63.20 {\com 1.00} & 68.30 {\com 0.50} & 76.90 {\com 0.70} & 80.60 {\com 0.40} & 69.00 {\com 1.00} & 60.10 {\com 2.90} & 72.63 {\com 0.57} \\
FedMosaic  & 74.20 {\com 1.40} & 86.80 {\com 1.20} & 61.90 {\com 7.10} & 65.80 {\com 0.60} & 74.30 {\com 0.70} & 77.70 {\com 1.50} & 68.20 {\com 0.20} & 60.20 {\com 3.00} & 71.14 {\com 0.98} \\
APFL  & 75.40 {\com 2.80} & 85.80 {\com 0.80} & 63.80 {\com 3.40} & 67.30 {\com 0.30} & 74.70 {\com 1.10} & 78.40 {\com 1.20} & 70.00 {\com 1.80} & 63.00 {\com 1.60} & 72.30 {\com 0.25} \\
PerAda  & 70.80 {\com 3.20} & 86.10 {\com 0.50} & 65.20 {\com 3.40} & 71.50 {\com 0.70} & 76.50 {\com 0.50} & 79.50 {\com 0.90} & 65.90 {\com 2.30} & 60.10 {\com 2.30} & 71.95 {\com 0.05} \\
\rowcolor{ours} \textsc{Potara}  & 75.20 {\com 1.28} & 87.80 {\com 1.31} & 67.87 {\com 0.57} & 71.73 {\com 2.23} & 78.47 {\com 1.04} & 78.80 {\com 2.36} & 69.00 {\com 1.18} & 64.53 {\com 2.00} & 74.18 {\com 0.29} \\
\bottomrule
\end{tabular}
}
\end{table*}

\begin{figure*}[!t]
    \centering
    \begin{subfigure}[t]{0.45\linewidth}
        \centering
        \includegraphics[width=\linewidth]{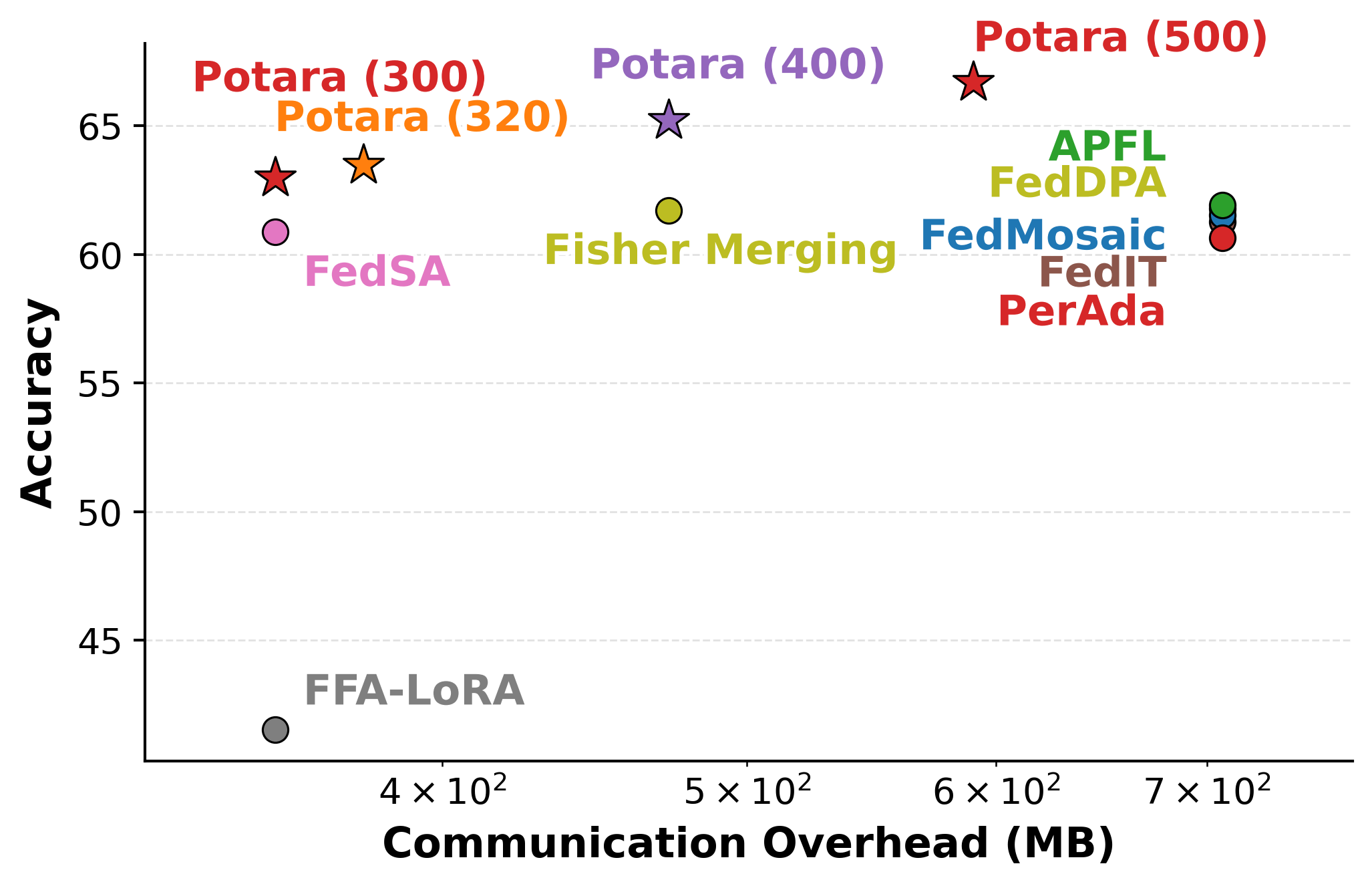}
        \caption{Personalized CIFAR-100 benchmark}
        \label{fig:cifar_acc_comm}
    \end{subfigure}
    \quad\quad
    \begin{subfigure}[t]{0.45\linewidth}
        \centering
        \includegraphics[width=\linewidth]{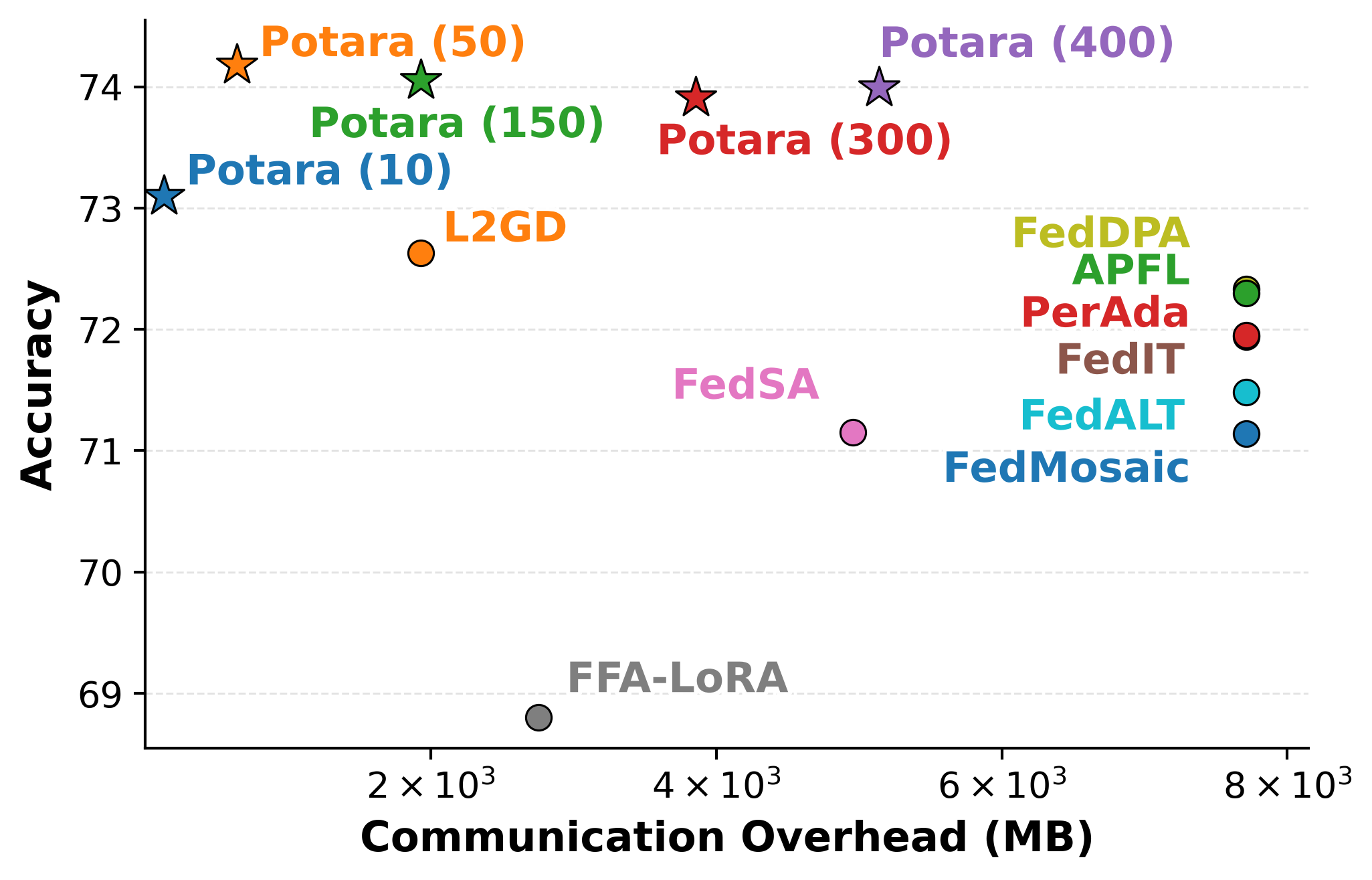}
        \caption{Commonsense reasoning benchmark}
        \label{fig:cs_acc_comm}
    \end{subfigure}
    \caption{
    Performance-communication trade-off across different benchmarks.
    \textsc{Potara} ($t$) indicates that the federated model used for merging is the FedIT model trained for $t$ rounds.
    {\color{black}For example, \textsc{Potara} (50) merges the FedIT model trained for 50 rounds and the local model.}
    \textsc{Potara} achieves strong performance with substantially lower communication overhead than the baselines.
    }
    \label{fig:acc-comm}
    \vspace{-4mm}
\end{figure*}

\subsection{Vision Benchmark}
We adopt ViT-B/16 \citep{dosovitskiy2021international} as the pre-trained model.
To simulate label heterogeneity, we partition CIFAR-100 \citep{krizhevsky2009learning} across clients using a Dirichlet distribution with concentration parameter $\alpha = 0.5$.
For each client, we construct both the training and test splits according to the same Dirichlet-based partition.
Unless otherwise specified, we consider $10$ clients.
All baselines are trained for $600$ rounds.
{\color{black}
For both Fisher Merging and the proposed \textsc{Potara}, we use the same total training budget as the baselines. Specifically, we first train the FedIT model for $300$ rounds and the local fine-tuning model for $300$ rounds, resulting in $600$ total training rounds, and then merge the two models to obtain the final personalized model.
As shown in Table~\ref{table_cifar100}, the proposed \textsc{Potara} achieves the best overall performance among all baseline methods under the same total training budget.}
Moreover, \textsc{Potara} reduces the communication cost.
See Section~\ref{acc_com_trade_off} and Figure~\ref{fig:cifar_acc_comm} for a detailed performance-communication trade-off analysis.

\vspace{-2mm}
\subsection{Language Benchmark}

We next evaluate \textsc{Potara} on LLM and NLP benchmark.
We adopt LLaMA-3.2-3B-Instruct~\citep{dubey2024llama} as the pre-trained model and construct a task-heterogeneous setting by assigning each client a distinct task from the commonsense reasoning benchmark~\citep{hu2023llm} and report accuracy for all tasks. 
We also evaluate \textsc{Potara} with Qwen3-4B on the FLAN benchmark, with the results provided in Appendix~\ref{appendix:larger_model}.
Each client has 1{,}000 training examples and 500 test examples. 
All federated baseline methods are trained for 600 rounds, while the local fine-tuning baseline is trained for 300 rounds to prevent overfitting. 
For \textsc{Potara}, we train FedIT for 50 rounds and local fine-tuning for 300 rounds, and then merge the resulting models to obtain the final personalized model for each client.
As reported in Table~\ref{table_cs}, \textsc{Potara} consistently outperforms all baselines across the evaluated NLP tasks.
{\color{black} Moreover, \textsc{Potara} uses substantially fewer training rounds than federated baselines while still achieving superior performance.
See Section~\ref{acc_com_trade_off} and Figure~\ref{fig:cs_acc_comm} for a detailed communication and computation efficiency analysis.}
The gains observed across both large vision (ViT) and language models (LLaMA) demonstrate that \textsc{Potara} generalizes well across model families, parameter scales, and task types.

\begin{table*}[t]
\centering
\scriptsize
\caption{Performance gain of model merging. \textsc{Potara} consistently improves over both endpoint models used for merging: the local model and the FedIT model.}
\label{tab:merge_gain}
\vspace{-1mm}

\begin{subtable}[t]{0.48\linewidth}
\centering
\caption{CIFAR-100}
\label{tab:merge_gain_cifar}
\resizebox{\linewidth}{!}{
\begin{tabular}{l|c|c|c}
\toprule
\textbf{Model} 
& \textbf{Accuracy ($\uparrow$)} 
& \textbf{Gain over Local} 
& \textbf{Gain over FedIT} \\
\midrule
Local (300 rounds) 
& $55.89 \pm 1.31$ 
& -- 
& -- \\
FedIT (300 rounds) 
& $55.42 \pm 1.78$ 
& -- 
& -- \\
\rowcolor{ours}
\textsc{Potara} 
& $\mathbf{62.98 \pm 1.47}$ 
& $\mathbf{+7.09}$ 
& $\mathbf{+7.56}$ \\
\bottomrule
\end{tabular}
}
\end{subtable}
\hfill
\begin{subtable}[t]{0.48\linewidth}
\centering
\caption{Commonsense reasoning}
\label{tab:merge_gain_cs}
\resizebox{\linewidth}{!}{
\begin{tabular}{l|c|c|c}
\toprule
\textbf{Model} 
& \textbf{Accuracy ($\uparrow$)} 
& \textbf{Gain over Local} 
& \textbf{Gain over FedIT} \\
\midrule
Local (300 rounds) 
& $72.05 \pm 0.25$ 
& -- 
& -- \\
FedIT (50 rounds) 
& $67.73 \pm 0.76$ 
& -- 
& -- \\
\rowcolor{ours}
\textsc{Potara} 
& $\mathbf{74.18 \pm 0.29}$ 
& $\mathbf{+2.13}$ 
& $\mathbf{+6.45}$ \\
\bottomrule
\end{tabular}
}
\end{subtable}

\vspace{-2mm}
\end{table*}

\begin{wraptable}{r}{0.6\textwidth}
\vspace{-4mm}
\caption{Computation Time and Performance Comparison.}
\label{tab:time_comparison}
\centering
\scriptsize
\resizebox{0.6\textwidth}{!}{
\begin{tabular}{l|ccccc|c}
\toprule
\textbf{Methods} 
& \begin{tabular}{c}\textbf{Federated}\\\textbf{Training}\end{tabular}
& \begin{tabular}{c}\textbf{Local}\\\textbf{Training}\end{tabular}
& \begin{tabular}{c}\textbf{FIM}\\\textbf{Estimation}\end{tabular}
& \begin{tabular}{c}\textbf{Covariance}\\\textbf{Estimation}\end{tabular}
& \begin{tabular}{c}\textbf{Total}\\\textbf{Time}\end{tabular}
& \begin{tabular}{c}\textbf{Avg.}\\\textbf{Acc.}\end{tabular} \\
\midrule
FedIT      & 5.92 h & N/A & N/A & N/A & 5.92 h & 71.94  \\
FedDPA     & 11.1 h & N/A & N/A & N/A & 11.1 h & 72.33  \\
FedALT     & 8.50 h & N/A & N/A & N/A & 8.50 h & 71.48 \\
APFL       & 11.0 h & N/A & N/A & N/A & 11.0 h & 72.30  \\
FedMosaic  & 6.33 h & N/A & N/A & N/A & 6.33 h & 71.14  \\
PerAda     & 11.3 h & N/A & N/A & N/A & 11.3 h & 71.95  \\
\rowcolor{ours}
\textsc{Potara} 
           & 0.49 h & 2.96 h & 0.43 h & 0.03 h & 3.91 h & 74.18  \\
\bottomrule
\end{tabular}
}
\vspace{-4mm}
\end{wraptable}

\vspace{-1mm}
\subsection{Communication and Computation Efficiency} \label{acc_com_trade_off}
In this subsection, we compare the communication and computation cost of \textsc{Potara} with baselines.
For ViT-B/16, LoRA introduces $0.59$ MB parameters for the $\mathbf{A}$ matrices and $0.59$ MB for the $\mathbf{B}$ matrices in float32 precision.
For LLaMA-3.2-3B-Instruct, LoRA introduces $8.26$ MB parameters for the $\mathbf{A}$ matrices and $4.59$ MB for the $\mathbf{B}$ matrices in float32 precision.
We use \textsc{Potara} ($t$) to denote our method where the FedIT model obtained at round $t$ is used for merging with the local model at round 300.
For example, \textsc{Potara} (50) merges the FedIT model trained for 50 rounds and local model trained for 300 rounds.
As shown in Fig.~\ref{fig:acc-comm}, \textsc{Potara} consistently achieves a favorable performance-communication trade-off across benchmarks.
Increasing the number of federated rounds leads to higher communication cost, but also improves the performance of \textsc{Potara} by providing a stronger federated model for merging.
On CIFAR-100, \textsc{Potara} (300) achieves around $63\%$ accuracy with only $354$ MB communication, outperforming most baselines clustered around $708$ MB.
With more federated rounds, \textsc{Potara} (400) and \textsc{Potara} (500) further improve the accuracy to $65.21\%$ and $66.70\%$, respectively.
On commonsense reasoning, \textsc{Potara} (50) reaches $74.18\%$ accuracy with only $642.5$ MB communication, outperforming baselines that require roughly $5$--$8\times 10^3$ MB communication.
Compared with high-communication baselines such as FedIT, FedALT, FedDPA, APFL, FedMosaic, and PerAda, \textsc{Potara} attains higher accuracy with substantially lower communication cost.
Compared with low-communication baselines such as FedSA, FFA-LoRA, and L2GD, \textsc{Potara} delivers clear performance gains at comparable or lower overhead.
These results demonstrate that \textsc{Potara} can achieve stronger performance with substantially lower communication cost across different benchmarks.
Table~\ref{tab:time_comparison} further compares the computation time on the commonsense reasoning benchmark.
Although \textsc{Potara} includes additional steps for local training, FIM estimation, and covariance estimation, its total computation time is only $3.91$ hours, which is lower than baselines. A detailed computation cost comparison is provided in Appendix~\ref{appendix:comnputation_comparison}.
These results demonstrate that \textsc{Potara} achieves superior performance while maintaining both communication and computation efficiency.

\vspace{-1mm}
\subsection{Performance Gain of Model Merging}

We further analyze the benefit of the model merging by comparing \textsc{Potara} with the two models used for merging: the FedIT model and the local model.
On CIFAR-100, \textsc{Potara} merges the FedIT model trained for $300$ rounds with the local model trained for $300$ rounds.
As shown in Table~\ref{tab:merge_gain_cifar}, the merged model achieves $62.98$ accuracy, improving over the local and FedIT models by $7.09$ and $7.56$ points, respectively.
On the commonsense reasoning benchmark, \textsc{Potara} merges the FedIT model trained for $50$ rounds with the local model trained for $300$ rounds, achieving $74.18$ accuracy and outperforming the local and FedIT models by $2.13$ and $6.45$ points, respectively, as shown in Table~\ref{tab:merge_gain_cs}.
These results show that model merging brings consistent gains over both individual models.
This suggests that the FedIT model and the local model capture complementary information: the former provides global knowledge obtained from collaborative training, while the latter preserves client-specific knowledge learned from local data.
By combining these two sources of knowledge, \textsc{Potara} achieves substantially better personalized performance than either model alone.

\begin{figure}[t]
    \centering
    \begin{subfigure}[t]{0.31\linewidth}
        \centering
    \includegraphics[width=\linewidth]{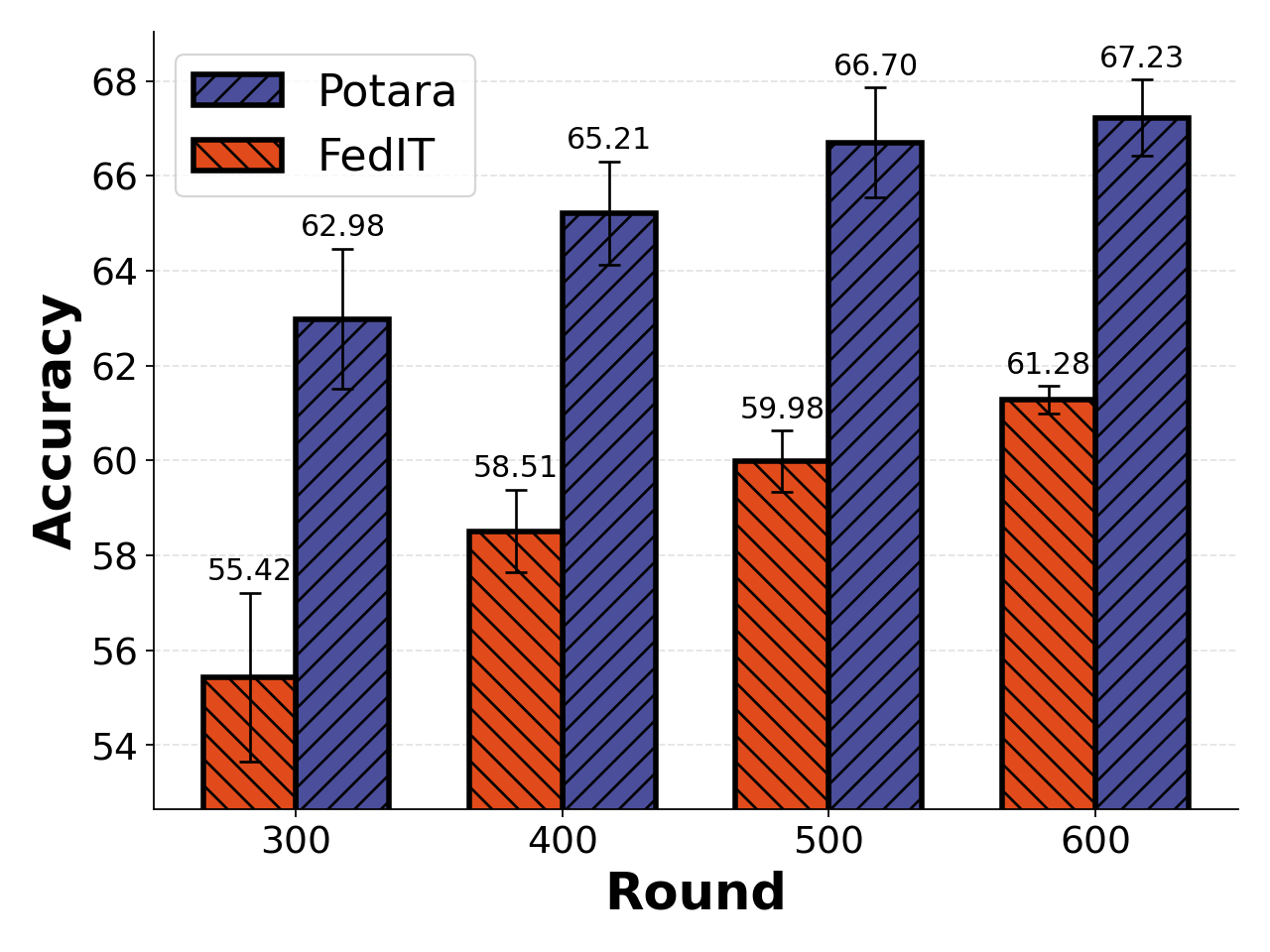}
        \caption{Effect of federated training round for merging (CIFAR-100).}
    \end{subfigure}
    \quad
    \begin{subfigure}[t]{0.31\linewidth}
        \centering        \includegraphics[width=\linewidth]{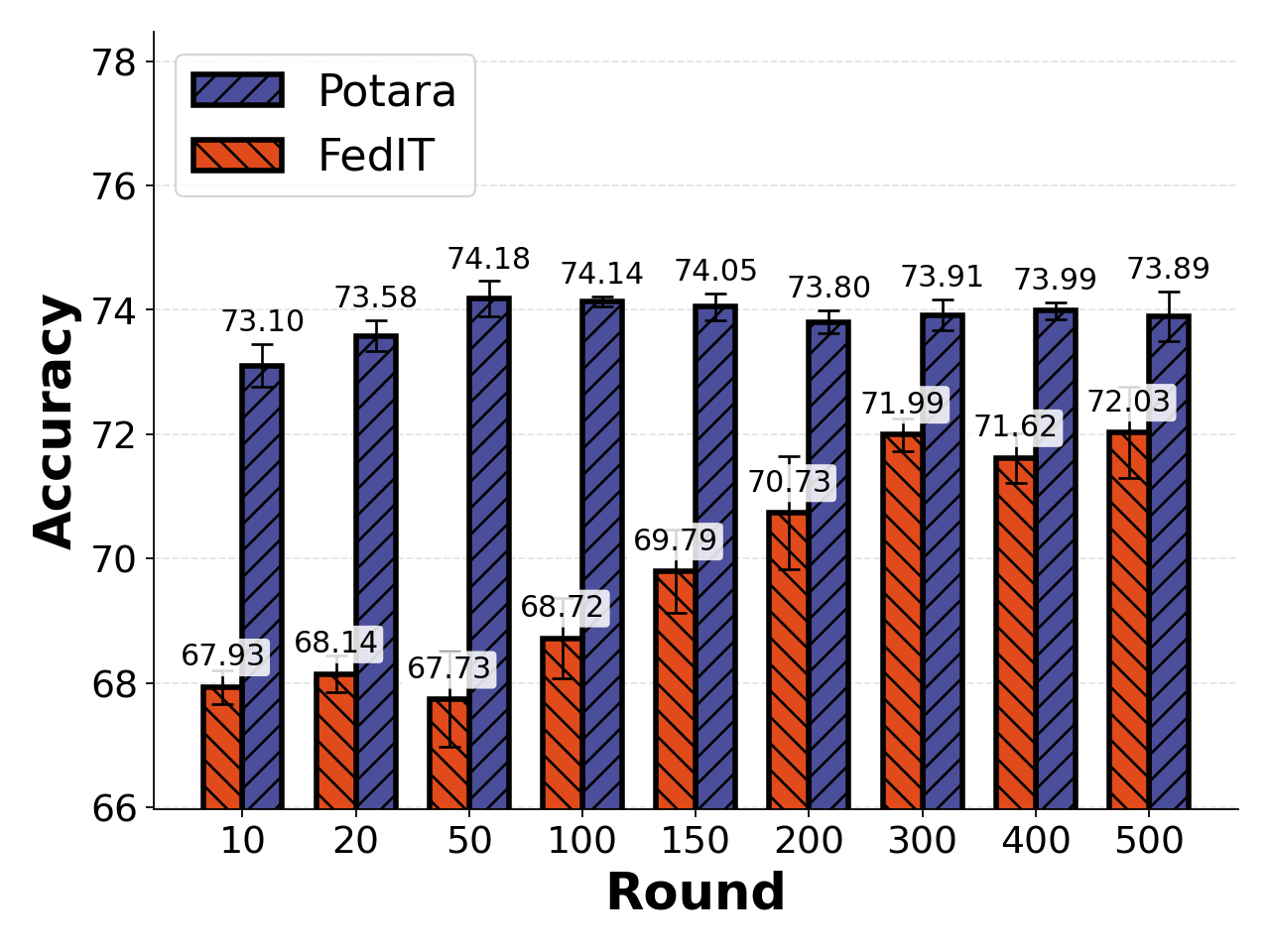}
        \caption{Effect of federated training round for merging (Commonsense Reasoning).}
    \end{subfigure}
    \quad
    \begin{subfigure}[t]{0.31\linewidth}
        \centering
        \includegraphics[width=\linewidth]{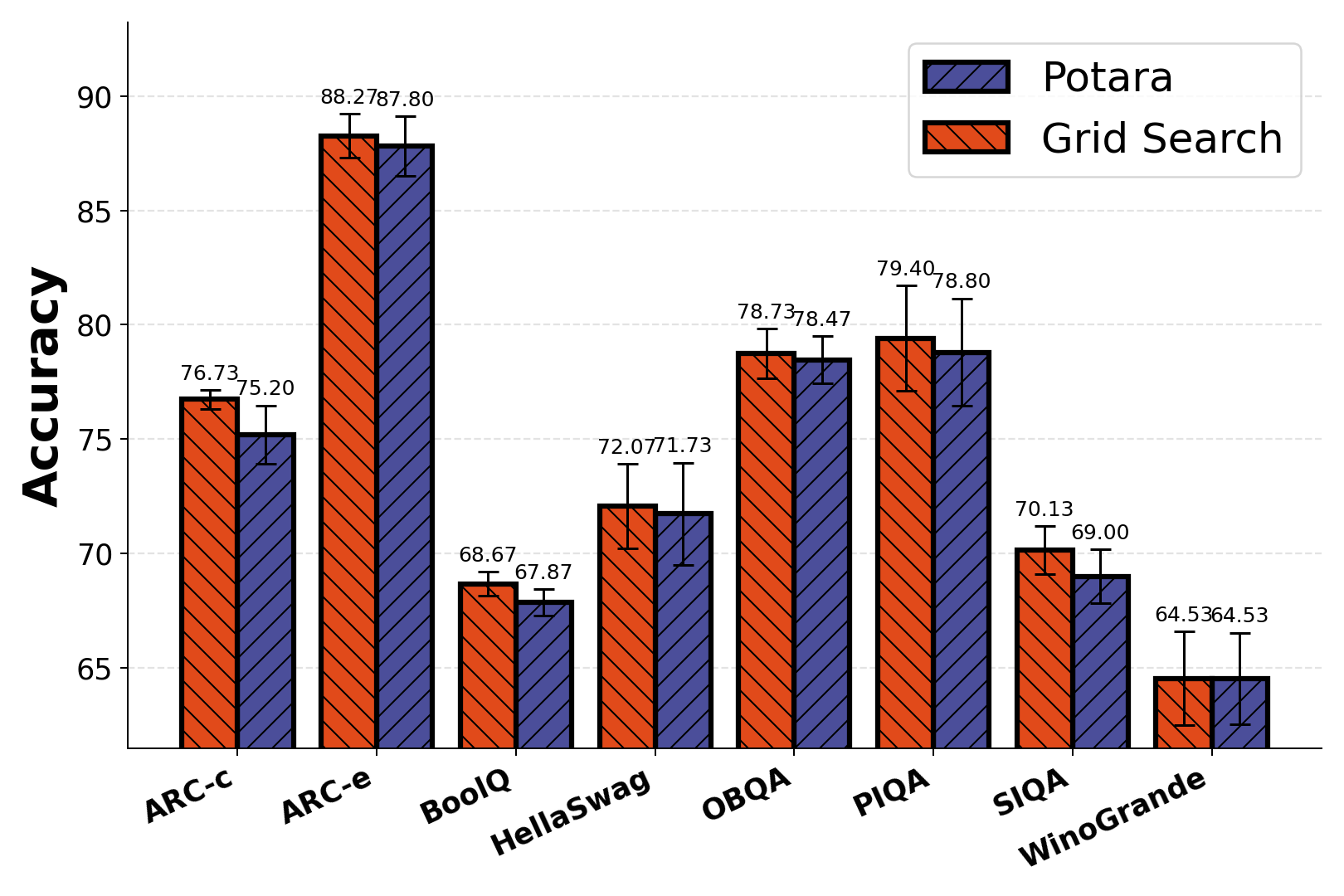}
    \caption{Client performance comparison between mixing weights computed by our method and mixing weights selected by grid search.}
    \end{subfigure}
    \caption{Ablation studies demonstrating the performance of \textsc{Potara}.}
    \label{fig:effect_round_fedit}
    \vspace{-5mm}
\end{figure}

\vspace{-3mm}
\subsection{Ablation Study}

\vspace{-1mm}
\noindent \textbf{Effect of federated training round for merging:}
We examine the effect of the federated training round used for merging.
Specifically, we merge the local model trained for 300 rounds with FedIT models trained for different numbers of rounds. 
As shown in Figure~\ref{fig:effect_round_fedit}a and Figure~\ref{fig:effect_round_fedit}b, \textsc{Potara} improves steadily as the FedIT model is trained longer on both benchmarks, indicating that a better-optimized FedIT model provides stronger general knowledge that can be more effectively combined with the local model.
On personalized CIFAR-100 benchmark (Figure~\ref{fig:effect_round_fedit}a), increasing the FedIT training rounds yields a clear monotonic gain for the merged model (e.g., from 62.98 to 67.23).
On the commonsense reasoning benchmark (Figure~\ref{fig:effect_round_fedit}b), the merged model already achieves strong performance even when using an early FedIT model (e.g., round 50), despite the FedIT still improving substantially with more rounds. 
This suggests that \textsc{Potara} can effectively extract useful general knowledge from intermediate FedIT models and combine it with local model.

\vspace{-1mm}
\noindent \textbf{Effect of mixing weight strategy:}
We compare client performance using the mixing weights computed by our method with the best-performing weights selected by grid search over $\{0, 0.1, \ldots, 1\}$.
As shown in Figure~\ref{fig:effect_round_fedit}c, the two bars are closely matched across clients, demonstrating that our computed weights achieve performance comparable to grid search while eliminating the expensive per-client hyperparameter sweep.

\vspace{-1mm}

\noindent \textbf{Further experiments:}
{\color{black}
Results with large-scale client settings are reported in Appendix~\ref{appendix:more_clients}.
Results with different LoRA ranks are reported in Appendix~\ref{appendix:different_ranks}.
Results with a larger pre-trained model and two additional benchmarks are reported in Appendix~\ref{appendix:larger_model} and Appendix \ref{appendix:domainnet}.
Across all these settings, \textsc{Potara} consistently achieves performance gains over the baselines.
}

\vspace{-3mm}
\section{Conclusion}
\label{section:conclusion}
We proposed \textsc{Potara}, a principled personalization framework for federated fine-tuning that merges a  federated model and a local model. 
We approximate the federated and local task vectors as jointly Gaussian random variables and derive the upper bound on the expected excess task loss of the federated model, the local model, and their merged model, where the dominant term is a trace of the task vector variance. 
By minimizing this bound, we obtain closed-form optimal mixing weights, guaranteeing that the merged model attains a tighter bound than both constituent models. 
Experiment results show that \textsc{Potara} consistently improves client performance across both vision and language benchmarks, outperforming baselines while retaining communication efficiency. 
{\color{black}
One limitation is that \textsc{Potara} focuses on merging two models.
Future work includes extending \textsc{Potara} to multi-model merging, and dynamic personalization settings where client distributions evolve over time.}

\section*{Acknowledgments}
This work was supported in part by the Office of Naval Research (ONR) under grants N00014-23-C-1016 and N00014-22-1-2305, the National Science Foundation (NSF) under grants CNS-2146171 and CPS-2313109, and by the Air Force Office of Scientific Research (AFOSR) under grant FA9550-24-1-0083.

\bibliography{ref.bib}

@STRING{ICLR = {International Conference on Learning Representations}}

@STRING{NeurIPS = {Advances in Neural Information Processing Systems}}

@STRING{ICML = {International Conference on Machine Learning}}

@STRING{CVPR = {IEEE/CVF Conference on Computer Vision and Pattern Recognition}}

@STRING{ICCV = {IEEE/CVF International Conference on Computer Vision}}

@article{peng2026hilora,
  title={HiLoRA: Hierarchical Low-Rank Adaptation for Personalized Federated Learning},
  author={Peng, Zihao and Zou, Nan and Zeng, Jiandian and Li, Guo and Chen, Ke and Li, Boyuan and Wang, Tian},
  journal=CVPR,
  year={2026}
}

@article{wang2024aggregation,
  title={An aggregation-free federated learning for tackling data heterogeneity},
  author={Wang, Yuan and Fu, Huazhu and Kanagavelu, Renuga and Wei, Qingsong and Liu, Yong and Goh, Rick Siow Mong},
  journal=CVPR,
  year={2024},
}

@article{yang2023efficient,
  title={Efficient model personalization in federated learning via client-specific prompt generation},
  author={Yang, Fu-En and Wang, Chien-Yi and Wang, Yu-Chiang Frank},
  journal=ICCV,
  year={2023},
}

@inproceedings{lee2025dynamic,
  title={Dynamic fisher-weighted model merging via bayesian optimization},
  author={Lee, Sanwoo and Liu, Jiahao and Wang, Qifan and Wang, Jingang and Cai, Xunliang and Wu, Yunfang},
  booktitle={Proceedings of the 2025 Conference of the Nations of the Americas Chapter of the Association for Computational Linguistics: Human Language Technologies (Volume 1: Long Papers)},
  pages={4923--4935},
  year={2025}
}

@inproceedings{daheimmodel,
  title={Model Merging by Uncertainty-Based Gradient Matching},
  author={Daheim, Nico and M{\"o}llenhoff, Thomas and Ponti, Edoardo and Gurevych, Iryna and Khan, Mohammad Emtiyaz},
  booktitle={The Twelfth International Conference on Learning Representations}
}

@inproceedings{ritter2018scalable,
  title={A scalable laplace approximation for neural networks},
  author={Ritter, Hippolyt and Botev, Aleksandar and Barber, David},
  booktitle={International conference on learning representations},
  year={2018}
}

@article{krizhevsky2009learning,
  title={Learning multiple layers of features from tiny images},
  author={Krizhevsky, Alex and Hinton, Geoffrey and others},
  year={2009},
  publisher={Toronto, ON, Canada}
}

@article{hanzely2020federated,
  title={Federated learning of a mixture of global and local models},
  author={Hanzely, Filip and Richt{\'a}rik, Peter},
  journal={arXiv preprint arXiv:2002.05516},
  year={2020}
}

@article{deng2020adaptive,
  title={Adaptive personalized federated learning},
  author={Deng, Yuyang and Kamani, Mohammad Mahdi and Mahdavi, Mehrdad},
  journal={arXiv preprint arXiv:2003.13461},
  year={2020}
}

@article{xie2024perada,
  title={Perada: Parameter-efficient federated learning personalization with generalization guarantees},
  author={Xie, Chulin and Huang, De-An and Chu, Wenda and Xu, Daguang and Xiao, Chaowei and Li, Bo and Anandkumar, Anima},
  journal={IEEE/CVF conference on computer vision and pattern recognition},
  pages={23838--23848},
  year={2024}
}

@article{seoco,
  title={Co-LoRA: Collaborative Model Personalization on Heterogeneous Multi-Modal Clients},
  author={Seo, Minhyuk and Kim, Taeheon and Lee, Hankook and Choi, Jonghyun and Tuytelaars, Tinne},
  journal=ICLR,
  year={2026}
}

@article{stoicamodel,
  title={Model merging with SVD to tie the Knots},
  author={Stoica, George and Ramesh, Pratik and Ecsedi, Boglarka and Choshen, Leshem and Hoffman, Judy},
  journal=ICLR,
  year={2025}
}

@article{qicabs,
  title={CABS: Conflict-Aware and Balanced Sparsification for Enhancing Model Merging},
  author={Qi, Binhang and Sun, Hailong and Long, Wenrui and Zhao, Ruobing and Gao, Xiang and others},
  journal=ICLR,
  year={2025}
}

@article{tammerging,
  title={Merging by Matching Models in Task Parameter Subspaces},
  author={Tam, Derek and Bansal, Mohit and Raffel, Colin},
  journal={Transactions on Machine Learning Research},
  year={2024}
}

@article{carmon2018accelerated,
  title={Accelerated methods for nonconvex optimization},
  author={Carmon, Yair and Duchi, John C and Hinder, Oliver and Sidford, Aaron},
  journal={SIAM Journal on Optimization},
  volume={28},
  number={2},
  pages={1751--1772},
  year={2018},
  publisher={SIAM}
}

@article{nesterov2006cubic,
  title={Cubic regularization of Newton method and its global performance},
  author={Nesterov, Yurii and Polyak, Boris T},
  journal={Mathematical programming},
  volume={108},
  number={1},
  pages={177--205},
  year={2006},
  publisher={Springer}
}

@article{zhou2023going,
  title={Going beyond linear mode connectivity: The layerwise linear feature connectivity},
  author={Zhou, Zhanpeng and Yang, Yongyi and Yang, Xiaojiang and Yan, Junchi and Hu, Wei},
  journal=NeurIPS,
  year={2023}
}

@article{hayou2024lora,
  title={Lora+: Efficient low rank adaptation of large models},
  author={Hayou, Soufiane and Ghosh, Nikhil and Yu, Bin},
  journal=ICML,
  year={2024}
}

@article{awais2025foundation,
  title={Foundation models defining a new era in vision: a survey and outlook},
  author={Awais, Muhammad and Naseer, Muzammal and Khan, Salman and Anwer, Rao Muhammad and Cholakkal, Hisham and Shah, Mubarak and Yang, Ming-Hsuan and Khan, Fahad Shahbaz},
  journal={IEEE Transactions on Pattern Analysis and Machine Intelligence},
  year={2025},
  publisher={IEEE}
}

@article{zhao2023survey,
  title={A survey of large language models},
  author={Zhao, Wayne Xin and Zhou, Kun and Li, Junyi and Tang, Tianyi and Wang, Xiaolei and Hou, Yupeng and Min, Yingqian and Zhang, Beichen and Zhang, Junjie and Dong, Zican and others},
  journal={arXiv preprint arXiv:2303.18223},
  volume={1},
  number={2},
  year={2023}
}

@article{wang2025more,
  title={Why Do More Experts Fail? A Theoretical Analysis of Model Merging},
  author={Wang, Zijing and Xu, Xingle and Liu, Yongkang and Zhang, Yiqun and Lin, Peiqin and Feng, Shi and Yang, Xiaocui and Wang, Daling and Sch{\"u}tze, Hinrich},
  journal={arXiv preprint arXiv:2505.21226},
  year={2025}
}

@article{ghosh2020efficient,
  title={An efficient framework for clustered federated learning},
  author={Ghosh, Avishek and Chung, Jichan and Yin, Dong and Ramchandran, Kannan},
  journal=NeurIPS,
  year={2020}
}

@article{duan2021flexible,
  title={Flexible clustered federated learning for client-level data distribution shift},
  author={Duan, Moming and Liu, Duo and Ji, Xinyuan and Wu, Yu and Liang, Liang and Chen, Xianzhang and Tan, Yujuan and Ren, Ao},
  journal={IEEE Transactions on Parallel and Distributed Systems},
  volume={33},
  number={11},
  pages={2661--2674},
  year={2021},
  publisher={IEEE}
}

@article{li2021ditto,
  title={Ditto: Fair and robust federated learning through personalization},
  author={Li, Tian and Hu, Shengyuan and Beirami, Ahmad and Smith, Virginia},
  journal=ICML,
  year={2021},
}

@article{smith2017federated,
  title={Federated multi-task learning},
  author={Smith, Virginia and Chiang, Chao-Kai and Sanjabi, Maziar and Talwalkar, Ameet S},
  journal=NeurIPS,
  year={2017}
}

@article{shamsian2021personalized,
  title={Personalized federated learning using hypernetworks},
  author={Shamsian, Aviv and Navon, Aviv and Fetaya, Ethan and Chechik, Gal},
  journal=ICML,
  year={2021},
}

@article{chen2018federated,
  title={Federated meta-learning with fast convergence and efficient communication},
  author={Chen, Fei and Luo, Mi and Dong, Zhenhua and Li, Zhenguo and He, Xiuqiang},
  journal={arXiv preprint arXiv:1802.07876},
  year={2018}
}

@article{tan2022towards,
  title={Towards personalized federated learning},
  author={Tan, Alysa Ziying and Yu, Han and Cui, Lizhen and Yang, Qiang},
  journal={IEEE transactions on neural networks and learning systems},
  volume={34},
  number={12},
  pages={9587--9603},
  year={2022},
  publisher={IEEE}
}

@article{pillutla2022federated,
  title={Federated learning with partial model personalization},
  author={Pillutla, Krishna and Malik, Kshitiz and Mohamed, Abdel-Rahman and Rabbat, Mike and Sanjabi, Maziar and Xiao, Lin},
  journal=ICML,
  year={2022},
}

@article{yu2024language,
  title={Language models are super mario: Absorbing abilities from homologous models as a free lunch},
  author={Yu, Le and Yu, Bowen and Yu, Haiyang and Huang, Fei and Li, Yongbin},
  journal=ICML,
  year={2024}
}

@article{yadav2023ties,
  title={Ties-merging: Resolving interference when merging models},
  author={Yadav, Prateek and Tam, Derek and Choshen, Leshem and Raffel, Colin A and Bansal, Mohit},
  journal=NeurIPS,
  year={2023}
}

@article{jin2022dataless,
  title={Dataless knowledge fusion by merging weights of language models},
  author={Jin, Xisen and Ren, Xiang and Preotiuc-Pietro, Daniel and Cheng, Pengxiang},
  journal=ICLR,
  year={2023}
}

@article{wortsman2022model,
  title={Model soups: averaging weights of multiple fine-tuned models improves accuracy without increasing inference time},
  author={Wortsman, Mitchell and Ilharco, Gabriel and Gadre, Samir Ya and Roelofs, Rebecca and Gontijo-Lopes, Raphael and Morcos, Ari S and Namkoong, Hongseok and Farhadi, Ali and Carmon, Yair and Kornblith, Simon and others},
  journal=ICML,
  year={2022},
}

@article{wang2024flora,
  title={Flora: Federated fine-tuning large language models with heterogeneous low-rank adaptations},
  author={Wang, Ziyao and Shen, Zheyu and He, Yexiao and Sun, Guoheng and Wang, Hongyi and Lyu, Lingjuan and Li, Ang},
  journal=NeurIPS,
  year={2024}
}

@article{bai2024federated,
  title={Federated fine-tuning of large language models under heterogeneous tasks and client resources},
  author={Bai, Jiamu and Chen, Daoyuan and Qian, Bingchen and Yao, Liuyi and Li, Yaliang},
  journal=NeurIPS,
  year={2024}
}

@article{cho2024heterogeneous,
  title={Heterogeneous lora for federated fine-tuning of on-device foundation models},
  author={Cho, Yae Jee and Liu, Luyang and Xu, Zheng and Fahrezi, Aldi and Joshi, Gauri},
  journal={Conference on Empirical Methods in Natural Language Processing},
  year={2024}
}

@article{chung2024scaling,
  title={Scaling instruction-finetuned language models},
  author={Chung, Hyung Won and Hou, Le and Longpre, Shayne and Zoph, Barret and Tay, Yi and Fedus, William and Li, Yunxuan and Wang, Xuezhi and Dehghani, Mostafa and Brahma, Siddhartha and others},
  journal={Journal of Machine Learning Research},
  volume={25},
  number={70},
  pages={1--53},
  year={2024}
}

@article{yang2025qwen3,
  title={Qwen3 technical report},
  author={Yang, An and Li, Anfeng and Yang, Baosong and Zhang, Beichen and Hui, Binyuan and Zheng, Bo and Yu, Bowen and Gao, Chang and Huang, Chengen and Lv, Chenxu and others},
  journal={arXiv preprint arXiv:2505.09388},
  year={2025}
}

@article{dosovitskiy2021international,
  title={An image is worth 16x16 words: Transformers for image recognition at scale},
  author={Dosovitskiy, Alexey and Beyer, Lucas and Kolesnikov, Alexander and Weissenborn, Dirk and Zhai, Xiaohua and Unterthiner, Thomas and Dehghani, Mostafa and Minderer, Matthias and Heigold, Georg and Gelly, Sylvain},
  journal=ICLR,
  year={2021}
}

@article{hao2025personalized,
  title={Personalized federated fine-tuning for heterogeneous data: An automatic rank learning approach via two-level lora},
  author={Hao, Jie and Wu, Yuman and Payani, Ali and Lee, Myungjin and Liu, Mingrui},
  journal={arXiv preprint arXiv:2503.03920},
  year={2025}
}

@article{collins2021exploiting,
  title={Exploiting shared representations for personalized federated learning},
  author={Collins, Liam and Hassani, Hamed and Mokhtari, Aryan and Shakkottai, Sanjay},
  journal=ICML,
  year={2021},
}

@article{fallah2020personalized,
  title={Personalized federated learning with theoretical guarantees: A model-agnostic meta-learning approach},
  author={Fallah, Alireza and Mokhtari, Aryan and Ozdaglar, Asuman},
  journal=NeurIPS,
  year={2020}
}

@article{tian2023distributed,
  title={Distributed learning over networks with graph-attention-based personalization},
  author={Tian, Zhuojun and Zhang, Zhaoyang and Yang, Zhaohui and Jin, Richeng and Dai, Huaiyu},
  journal={IEEE Transactions on Signal Processing},
  volume={71},
  pages={2071--2086},
  year={2023},
  publisher={IEEE}
}

@article{fisher1922mathematical,
  title={On the mathematical foundations of theoretical statistics},
  author={Fisher, Ronald A},
  journal={Philosophical transactions of the Royal Society of London. Series A, containing papers of a mathematical or physical character},
  volume={222},
  number={594-604},
  pages={309--368},
  year={1922},
  publisher={The Royal Society London}
}

@article{daxberger2021laplace,
  title={Laplace redux-effortless bayesian deep learning},
  author={Daxberger, Erik and Kristiadi, Agustinus and Immer, Alexander and Eschenhagen, Runa and Bauer, Matthias and Hennig, Philipp},
  journal=NeurIPS,
  year={2021}
}

@article{mackay1992practical,
  title={A practical Bayesian framework for backpropagation networks},
  author={MacKay, David JC},
  journal={Neural computation},
  volume={4},
  number={3},
  pages={448--472},
  year={1992},
  publisher={MIT Press One Rogers Street, Cambridge, MA 02142-1209, USA journals-info~…}
}

@inproceedings{mcmahan2017communication,
  title={Communication-efficient learning of deep networks from decentralized data},
  author={McMahan, Brendan and Moore, Eider and Ramage, Daniel and Hampson, Seth and y Arcas, Blaise Aguera},
  booktitle={Artificial intelligence and statistics},
  pages={1273--1282},
  year={2017},
  organization={PMLR}
}

@article{neyshabur2020being,
  title={What is being transferred in transfer learning?},
  author={Neyshabur, Behnam and Sedghi, Hanie and Zhang, Chiyuan},
  journal=NeurIPS,
  year={2020}
}

@article{nagarajan2019uniform,
  title={Uniform convergence may be unable to explain generalization in deep learning},
  author={Nagarajan, Vaishnavh and Kolter, J Zico},
  journal=NeurIPS,
  year={2019}
}

@article{draxler2018essentially,
  title={Essentially no barriers in neural network energy landscape},
  author={Draxler, Felix and Veschgini, Kambis and Salmhofer, Manfred and Hamprecht, Fred},
  journal=ICML,
  year={2018}
}

@article{frankle2020linear,
  title={Linear mode connectivity and the lottery ticket hypothesis},
  author={Frankle, Jonathan and Dziugaite, Gintare Karolina and Roy, Daniel and Carbin, Michael},
  journal=ICML,
  year={2020}
}

@article{garipov2018loss,
  title={Loss surfaces, mode connectivity, and fast ensembling of dnns},
  author={Garipov, Timur and Izmailov, Pavel and Podoprikhin, Dmitrii and Vetrov, Dmitry P and Wilson, Andrew G},
  journal=NeurIPS,
  year={2018}
}

@article{qi2024fdlora,
  title={{FDLoRA}: Personalized federated learning of large language model via dual lora tuning},
  author={Qi, Jiaxing and Luan, Zhongzhi and Huang, Shaohan and Fung, Carol and Yang, Hailong and Qian, Depei},
  journal={arXiv preprint arXiv:2406.07925},
  year={2024}
}

@article{pascanu2013revisiting,
  title={Revisiting natural gradient for deep networks},
  author={Pascanu, Razvan and Bengio, Yoshua},
  journal={arXiv preprint arXiv:1301.3584},
  year={2013}
}

@article{matena2022merging,
  title={Merging models with fisher-weighted averaging},
  author={Matena, Michael S and Raffel, Colin A},
  journal=NeurIPS,
  year={2022}
}

@article{sun2024improving,
  title={Improving Lo{RA} in Privacy-preserving Federated Learning},
  author={Sun, Youbang and Li, Zitao and Li, Yaliang and Ding, Bolin},
  journal=ICLR,
  year={2024}
}

@article{bian2025fedalt,
  title={FedALT: Federated Fine-Tuning through Adaptive Local Training with Rest-of-World LoRA},
  author={Bian, Jieming and Wang, Lei and Zhang, Letian and Xu, Jie},
  journal={AAAI Conference on Artificial Intelligence},
  year={2026}
}

@article{
ilharco2023editing,
title={Editing models with task arithmetic},
  author={Ilharco, Gabriel and Ribeiro, Marco Tulio and Wortsman, Mitchell and Gururangan, Suchin and Schmidt, Ludwig and Hajishirzi, Hannaneh and Farhadi, Ali},
journal=ICLR,
year={2023},
}

@article{yang2024dual,
 author = {Yang, Yiyuan and Long, Guodong and Shen, Tao and Jiang, Jing and Blumenstein, Michael},
 title = {Dual-Personalizing Adapter for Federated Foundation Models},
 journal = NeurIPS,
 year = {2024}
}

@article{chen2025fedmerge,
  title={FedMerge: Federated Personalization via Model Merging},
  author={Chen, Shutong and Zhou, Tianyi and Long, Guodong and Jiang, Jing and Zhang, Chengqi},
  journal={AAAI Conference on Artificial Intelligence},
  year={2026}
}

@inproceedings{zhang2024towards,
  title={Towards building the federatedgpt: Federated instruction tuning},
  author={Zhang, Jianyi and Vahidian, Saeed and Kuo, Martin and Li, Chunyuan and Zhang, Ruiyi and Yu, Tong and Wang, Guoyin and Chen, Yiran},
  booktitle={ICASSP 2024-2024 IEEE International Conference on Acoustics, Speech and Signal Processing (ICASSP)},
  pages={6915--6919},
  year={2024},
  organization={IEEE}
}

@article{guo2024selective,
  title={Selective Aggregation for Low-Rank Adaptation in Federated Learning},
  author={Guo, Pengxin and Zeng, Shuang and Wang, Yanran and Fan, Huijie and Wang, Feifei and Qu, Liangqiong},
  journal=ICLR,
  year={2025}
}

@article{hu2022lora,
  title={Lora: Low-rank adaptation of large language models.},
  author={Hu, Edward J and Shen, Yelong and Wallis, Phillip and Allen-Zhu, Zeyuan and Li, Yuanzhi and Wang, Shean and Wang, Lu and Chen, Weizhu and others},
  journal=ICLR,
  year={2022}
}

@article{fang2025federated,
  title={Federated sketching lora: On-device collaborative fine-tuning of large language models},
  author={Fang, Wenzhi and Han, Dong-Jun and Yuan, Liangqi and Hosseinalipour, Seyyedali and Brinton, Christopher G},
  journal={arXiv preprint arXiv:2501.19389},
  year={2025}
}

@article{zou2025joint,
  title={Joint rank optimization and bandwidth allocation for heterogeneous federated lora fine-tuning},
  author={Zou, Yinan and Wang, Zixin and Shi, Yuanming and Zhou, Yong},
  journal={IEEE Transactions on Vehicular Technology},
  year={2025},
  publisher={IEEE}
}

@article{dubey2024llama,
  title={The llama 3 herd of models},
  author={Dubey, Abhimanyu and Jauhri, Abhinav and Pandey, Abhinav and Kadian, Abhishek and Al-Dahle, Ahmad and Letman, Aiesha and Mathur, Akhil and Schelten, Alan and Yang, Amy and Fan, Angela and others},
  journal={arXiv e-prints},
  pages={arXiv--2407},
  year={2024}
}

@article{hu2023llm,
  title={Llm-adapters: An adapter family for parameter-efficient fine-tuning of large language models},
  author={Hu, Zhiqiang and Wang, Lei and Lan, Yihuai and Xu, Wanyu and Lim, Ee-Peng and Bing, Lidong and Xu, Xing and Poria, Soujanya and Lee, Roy Ka-Wei},
  journal={arXiv preprint arXiv:2304.01933},
  year={2023}
}
\bibliographystyle{abbrvnat}

\newpage
\appendix

\begin{center}
    {\bf\Large Appendix}
\end{center}

\startcontents[sections]
\printcontents[sections]{l}{1}{\setcounter{tocdepth}{3}}

\newpage
\section{Theoretical Results}
\label{appendix:theory_proof}

\subsection{Derivation of Gaussian Distribution Approximation} \label{appendix:A1}

The gaussian distribution approximation has been used in prior works \citep{matena2022merging,tammerging}, but the full derivation is omitted. 
For completeness, we provide the full derivation.

    The training of the individual task vector for task $i$ can be formulated as a maximum a posteriori (MAP) estimation problem \citep{matena2022merging,wang2025more} as follows 
\begin{align}
    \bm{\mu}_i \, = \, \arg \mathop{\max}_{\bm{\theta}_i} \, p(\bm{\theta}_i \, | \, \mathcal{D}_i),
\end{align}
where $\mathcal{D}_i$ denotes the dataset associated with task $i$. 
Here, task $i$ may correspond either to the FedIT objective or to the local fine-tuning objective.
 
We perform a second-order Taylor expansion of the log-posterior around a point  $\hat{\bm{\theta}}$:
\begin{align}
    \log p(\bm{\theta}_i \, | \, \mathcal{D}_i)
    &\approx 
     \log p(\hat{\bm{\theta}} \, | \, \mathcal{D}_i) 
      + (\bm{\theta}_i-\hat{\bm{\theta}})^\top 
     \nabla_{\bm{\theta}_i}  \log p(\bm{\theta}_i \, | \, \mathcal{D}_i) |_{\bm{\theta}_i = \hat{\bm{\theta}}}  
      - \frac{1}{2} 
     (\bm{\theta}_i-\hat{\bm{\theta}})^\top
     \mathbf{H}_i
     (\bm{\theta}_i-\hat{\bm{\theta}}),
\end{align}
where $\mathbf{H}_i = - \nabla^2_{\bm{\theta}_i} \log p(\bm{\theta}_i \, | \, \mathcal{D}_i) |_{\bm{\theta}_i = \hat{\bm{\theta}}}$.

We assume that the parameter $\bm{\mu}_i$ of 
the trained neural network is a local maximum of the posterior distribution. 
By setting $\hat{\bm{\theta}} = \bm{\mu}_i$, we have  $\nabla_{\bm{\theta}_i}  \log p(\bm{\theta}_i \, | \, \mathcal{D}_i) |_{\bm{\theta}_i = \bm{\mu}_i } = 0$.

Then, we obtain
\begin{align}
    \log p(\bm{\theta}_i \, | \, \mathcal{D}_i)
    \approx
     \log p(\bm{\mu}_i \, | \, \mathcal{D}_i)
     - \frac{1}{2} 
     (\bm{\theta}_i-\bm{\mu}_i)^\top
     \mathbf{H}_i
     (\bm{\theta}_i-\bm{\mu}_i),
\end{align}
where $\mathbf{H}_i = - \nabla^2_{\bm{\theta}_i} \log p(\bm{\theta}_i \, | \, \mathcal{D}_i) |_{\bm{\theta}_i = \bm{\mu}_i}$.

Applying the Laplace approximation \citep{mackay1992practical,daxberger2021laplace} around the local optimum $\bm{\mu}_i$, the posterior distribution $p(\bm{\theta}_i \, | \, \mathcal{D}_i)$  admits the following Gaussian approximation:
\begin{align}
    p(\bm{\theta}_i \, | \, \mathcal{D}_i) \approx \mathcal{N}(\bm{\theta}_i;\bm{\mu}_i, \mathbf{H}_i^{-1})
\end{align}
Equivalently, we write
\begin{align}
    \bm{\theta}_i  \sim \mathcal{N}(\bm{\mu}_i, \bm{\Sigma}_i) \quad \text{where}\, \bm{\Sigma}_i = \mathbf{H}_i^{-1}.
\end{align}

Given a trained neural network model $p_{\bm{\mu}_i}(\bm{y} \mid \bm{x})$, the Fisher information matrix \citep{matena2022merging,fisher1922mathematical} is defined as
\begin{align}
\mathbf{F}_{\bm{\mu}_i} 
=
\mathbb{E}_{\bm{x}}
\Big[
\mathbb{E}_{\bm{y}  \sim p_{\bm{\mu}_i}\!(\bm{y} \mid \bm{x})}
\big[
\nabla_{\bm{\mu}_i} \log p_{\bm{\mu}_i}(\bm{y} \mid \bm{x})
\nabla_{\bm{\mu}_i} \log p_{\bm{\mu}_i}(\bm{y} \mid \bm{x})^\top
\big]
\Big].
\end{align}
It can be shown that the Fisher information matrix coincides with the Hessian at modes of the distribution \citep{pascanu2013revisiting}. 

\newpage

\newpage

\subsection{Proof of Theorem \ref{loss_bound}} \label{appendix:loss_bound}

Recall that the joint distribution of $(\bm{\theta}_{\mathrm{FedIT}},\bm{\theta}_{\mathrm{Local}})$ is given by
\begin{align}
         \begin{bmatrix}
         \bm{\theta}_{\mathrm{FedIT}}
         \\ \bm{\theta}_{\mathrm{Local}}
        \end{bmatrix} 
        \sim \mathcal{N} \bigg( \begin{bmatrix}
         \bm{\mu}_{\mathrm{FedIT}}
         \\ \bm{\mu}_{\mathrm{Local}}
        \end{bmatrix},
        \begin{bmatrix}
            \bm{\Sigma}_{\mathrm{FedIT}} &\bm{\Sigma}_{\mathrm{Cross}}
           \\  \bm{\Sigma}_{\mathrm{Cross}}^\top & \bm{\Sigma}_{\mathrm{Local}}
        \end{bmatrix} \bigg).
\end{align}
where $\bm{\Sigma}_{\mathrm{Cross}}$ denotes the cross-covariance between the two task vectors with
$\bm{\Sigma}_{\mathrm{Cross}} \succeq \mathbf{0}$, $\bm{\Sigma}_{\mathrm{Cross}} \preceq \bm{\Sigma}_{\mathrm{FedIT}}$ and
$\bm{\Sigma}_{\mathrm{Cross}} \preceq \bm{\Sigma}_{\mathrm{Local}}$.

Then, the merged task vector follows the following Gaussian distribution:
\begin{align}
    \bm{\theta}_{\mathrm{Merge}} 
    \sim \mathcal{N}(\bm{\mu}_{\mathrm{Merge}}, \bm{\Sigma}_{\mathrm{Merge}})
\end{align}
where
\begin{align}
&\bm{\mu}_{\mathrm{merge}} = \lambda_{\mathrm{FedIT}} \bm{\mu}_{\mathrm{FedIT}} + \lambda_{\mathrm{Local}} \bm{\mu}_{\mathrm{Local}},
\\
&\bm{\Sigma}_{\mathrm{merge}} 
= \lambda_{\mathrm{FedIT}}^2 \bm{\Sigma}_{\mathrm{FedIT}} + \lambda_{\mathrm{Local}}^2 \bm{\Sigma}_{\mathrm{Local}}
   + \lambda_{\mathrm{FedIT}}\lambda_{\mathrm{Local}}(\bm{\Sigma}_{\mathrm{Cross}}+\bm{\Sigma}_{\mathrm{Cross}}^{\!\top}),
\\& \lambda_{\mathrm{FedIT}},\lambda_{\mathrm{Local}} \ge 0,
\\& \lambda_{\mathrm{FedIT}} + \lambda_{\mathrm{Local}} = 1.
\end{align}

To analyze the effect of model merging, we expand the task loss $\mathcal L$ around its optimum $\mathbf W^*$.
Let $\bm\phi \;=\; \mathbf W_{\mathrm{Pre}}+\bm\theta_{\mathrm{Merge}}-\mathbf W^*$ and $
\mathbf H \;=\; \nabla^2\mathcal L(\mathbf W^*)$, where $\mathbf{W}_{\mathrm{Pre}}$ is the pre-trained model that can be regarded as a constant vector.
By Assumption \ref{assump:L_H} and Lemma 2.2 in \cite{carmon2018accelerated}, we have
\begin{align}
    \mathcal L(\mathbf W_{\mathrm{Pre}}+\bm\theta_{\mathrm{Merge}})
    \leq
    \mathcal L(\mathbf W^*)
    +\frac12\,\bm\phi^\top \mathbf H\,\bm\phi
    + \frac{L_H}{6} \|\bm\phi\|_2^3.
\end{align}

By taking expectation, we obtain
\begin{align}
    \mathbb{E}\big[\mathcal L(\mathbf{W}_{\mathrm{Pre}}+ \bm{\theta}_{\mathrm{Merge}})-\mathcal L(\mathbf{W}^*)\big]
\leq \frac12\,\mathbb{E}\!\big[\bm{\phi}^\top \mathbf{H} \bm{\phi} \big] + \frac{L_H}{6}\mathbb{E}[\|\bm\phi\|_2^3].
\end{align}

Using Lemma~\ref{identity} (i.e.,
$\mathbb{E}[\bm{\phi}^\top \mathbf{H} \bm{\phi}]
=(\mathbb{E}\bm{\phi})^\top \mathbf{H}(\mathbb{E}\bm{\phi})
+\mathrm{tr}(\mathbf{H}\,\mathrm{Cov}(\bm{\phi}))$),
we obtain
\begin{align}
&\mathbb{E}\big[\mathcal L(\mathbf{W}_{\mathrm{Pre}}+ \bm{\theta}_{\mathrm{Merge}})
-\mathcal L(\mathbf{W}^*)\big] \notag
\\ &
\leq \frac12\,\mathbb{E}\!\big[\bm{\phi}^\top \mathbf{H}\bm{\phi}\big]
+ \frac{L_H}{6}\mathbb{E}[\|\bm\phi\|_2^3] \notag\\ 
& = 
\frac12 (\mathbb{E}\bm{\phi})^\top \mathbf{H}(\mathbb{E}\bm{\phi})
+ \frac12 \mathrm{tr}(\mathbf{H}\,\mathrm{Cov}(\bm{\phi})) 
+ \frac{L_H}{6}\mathbb{E}[\|\bm\phi\|_2^3]  \notag \\
&=
\frac12
(\mathbf{W}_{\mathrm{Pre}} + \bm{\mu}_{\mathrm{Merge}}-\mathbf{W}^*)^\top
\mathbf{H}
(\mathbf{W}_{\mathrm{Pre}} + \bm{\mu}_{\mathrm{Merge}}-\mathbf{W}^*)
+\frac12 \mathrm{tr}(\mathbf{H}\,\bm{\Sigma}_{\mathrm{Merge}})
+ \frac{L_H}{6} \mathbb{E}\big[  \|\bm\phi\|_2^3  \big],  \label{eq:bound_step1}
\end{align}
where $\bm{\mu}_{\mathrm{Merge}}=\mathbb{E}[\bm{\theta}_{\mathrm{Merge}}]$ and 
$\bm{\Sigma}_{\mathrm{Merge}}=\mathrm{Cov}(\bm{\theta}_{\mathrm{Merge}})$.

By Assumption \ref{assump:basin_event}, we have $\|\bm\phi\|_2\le \delta$,
and therefore $\|\bm\phi\|_2^3 \le \delta\,\|\bm\phi\|_2^2$.
Taking expectation yields $\mathbb{E}\|\bm{\phi}\|_2^3
\le \delta\,\mathbb{E}\|\bm{\phi}\|_2^2$.
By the second-moment decomposition, we have
\begin{align}
\mathbb{E}\|\bm{\phi}\|_2^2
=
\|\mathbb{E}\bm{\phi}\|_2^2 + \mathrm{tr}(\mathrm{Cov}(\bm{\phi}))
=
\|\mathbf{W}_{\mathrm{Pre}} + \bm{\mu}_{\mathrm{Merge}}-\mathbf{W}^*\|_2^2
+\mathrm{tr}(\bm{\Sigma}_{\mathrm{Merge}}). \label{eq:second_moment}
\end{align}
Combining \eqref{eq:bound_step1}--\eqref{eq:second_moment}, we conclude that
\begin{align}
&\mathbb{E}\big[\mathcal L(\mathbf{W}_{\mathrm{Pre}}+ \bm{\theta}_{\mathrm{Merge}})
-\mathcal L(\mathbf{W}^*)\big] \notag
\\&\le
\frac12
(\mathbf{W}_{\mathrm{Pre}} + \bm{\mu}_{\mathrm{Merge}}-\mathbf{W}^*)^\top
\mathbf{H}
(\mathbf{W}_{\mathrm{Pre}} + \bm{\mu}_{\mathrm{Merge}}-\mathbf{W}^*)
+\frac12 \mathrm{tr}(\mathbf{H}\,\bm{\Sigma}_{\mathrm{Merge}})
\notag \\
&\quad
+\frac{L_H}{6} \delta\,
\Big(
\|\mathbf{W}_{\mathrm{Pre}} + \bm{\mu}_{\mathrm{Merge}}-\mathbf{W}^*\|_2^2
+\mathrm{tr}(\bm{\Sigma}_{\mathrm{Merge}})
\Big). \label{eq:final_bound}
\end{align}

Then, we have 
\begin{align}
    \mathbb{E}\big[\mathcal L(\mathbf{W}_{\mathrm{Pre}}+ \bm{\theta}_{\mathrm{Merge}})
    -\mathcal L(\mathbf{W}^*)\big]
    &\le \frac{L}{2} \| \mathbf{W}_{\mathrm{Pre}} + \bm{\mu}_{\mathrm{Merge}}-\mathbf{W}^* \|^2 
    + \frac{L}{2}  \mathrm{tr}(\bm{\Sigma}_{\mathrm{Merge}}) \notag
    \\& \quad +\frac{L_H}{6} \delta\,
    \Big(
    \|\mathbf{W}_{\mathrm{Pre}} + \bm{\mu}_{\mathrm{Merge}}-\mathbf{W}^*\|_2^2
    +\mathrm{tr}(\bm{\Sigma}_{\mathrm{Merge}})
    \Big) \notag
    \\& \le \frac{L}{2} \delta^2 + \frac{L}{2}  \mathrm{tr}(\bm{\Sigma}_{\mathrm{Merge}})
    + \frac{L_H}{6} \delta^3 + \frac{L_H}{6} \delta \mathrm{tr}(\bm{\Sigma}_{\mathrm{Merge}}) \notag
    \\& =  \bigg(\frac{L}{2} + \frac{L_H}{6}\delta \bigg)\delta^2 + \bigg( \frac{L}{2} + \frac{L_H}{6}\delta \bigg) \mathrm{tr}(\bm{\Sigma}_{\mathrm{Merge}})
\end{align}
where the first inequality is due to Assumption \ref{smoothness}, and the second inequality is due to Assumption \ref{assump:basin_event}.

Similarly, for each individual task vector we have
\begin{align}
    \mathbb{E}\big[\mathcal L(\mathbf{W}_{\mathrm{Pre}}+ \bm{\theta}_{i})-\mathcal L(\mathbf{W}^*)\big]
    \leq
    \bigg(\frac{L}{2} + \frac{L_H}{6}\delta \bigg)\delta^2 + \bigg( \frac{L}{2} + \frac{L_H}{6}\delta \bigg) \mathrm{tr}(\bm{\Sigma}_{i}), \quad i \in \{\mathrm{FedIT},\mathrm{Local}\}.
\end{align}

\subsection{Proof of Lemma \ref{corollary}}

Let 
\begin{align}
    a = \mathrm{tr}(\bm{\Sigma}_{\mathrm{FedIT}}), 
\quad 
b = \mathrm{tr}(\bm{\Sigma}_{\mathrm{Local}}), 
\quad 
c = \mathrm{tr}(\bm{\Sigma}_{\mathrm{Cross}}).
\end{align}
Hence
\begin{align}
\mathrm{tr}(\bm{\Sigma}_{\mathrm{Merge}})
& = \mathrm{tr}(\lambda_{\mathrm{FedIT}}^2 \bm{\Sigma}_{\mathrm{FedIT}} + \lambda_{\mathrm{Local}}^2 \cdot \bm{\Sigma}_{\mathrm{Local}} + \lambda_{\mathrm{FedIT}} \lambda_{\mathrm{Local}} (\bm{\Sigma}_{\mathrm{Cross}} + \bm{\Sigma}_{\mathrm{Cross}}^{\top}) ) \notag
\\&= \lambda_{\mathrm{FedIT}}^2 \cdot a + \lambda_{\mathrm{Local}}^2 \cdot b + 2\lambda_{\mathrm{FedIT}} \lambda_{\mathrm{Local}} \cdot c 
\end{align}

Using the Lemma \ref{lemma2},
we have
\begin{align}
    |c| \;=\; |\mathrm{tr}(\bm{\Sigma}_{\mathrm{Cross}})|
\;\le\;
\sqrt{\mathrm{tr}(\bm{\Sigma}_{\mathrm{FedIT}})\,
       \mathrm{tr}(\bm{\Sigma}_{\mathrm{Local}})}
\;=\;
\sqrt{ab}.
\end{align}

We define a function $g(\cdot)$ with respect to $\lambda_{\mathrm{FedIT}}$ as follows
\begin{align}
    g(\lambda_{\mathrm{FedIT}}) = \mathrm{tr}(\bm{\Sigma}_{\mathrm{Merge}})
= a\lambda_{\mathrm{FedIT}}^2 + b(1-\lambda_{\mathrm{FedIT}})^2 + 2\lambda_{\mathrm{FedIT}}(1-\lambda_{\mathrm{FedIT}})c,
\end{align}
where $0 \leq \lambda_{\mathrm{FedIT}} \leq 1$.

This quadratic form is convex in $\lambda_{\mathrm{FedIT}}$,
because $g^{''}(\lambda_{\mathrm{FedIT}}) = 2(a+b-2c) \geq 2(a+b-2\sqrt{ab}) = 2(\sqrt{a}-\sqrt{b})^2 \geq 0$.

Then, we obtain the optimal $ \lambda_{\mathrm{FedIT}}$ and the optimum as 
\begin{align}
    \lambda_{\mathrm{FedIT}}^{*}
    = \frac{b - c}{a + b - 2c} \in \, [0,1],
\end{align}
\begin{align}
    g(\lambda_{\mathrm{FedIT}}^{*}) = b - \frac{(b - c)^2}{a+b-2c} = a - \frac{(a -c)^2}{a+b-2c}.
\end{align}

Due to $\bm{\Sigma}_{\mathrm{Cross}} \preceq \bm{\Sigma}_\mathrm{FedIT}$ and $\bm{\Sigma}_{\mathrm{Cross}} \preceq \bm{\Sigma}_\mathrm{Local}$, it follows that $\mathrm{tr}(\bm{\Sigma}_{\mathrm{Cross}}) \le \mathrm{tr}(\bm{\Sigma}_{\mathrm{FedIT}})$ and $\mathrm{tr}(\bm{\Sigma}_{\mathrm{Cross}}) \le \mathrm{tr}(\bm{\Sigma}_{\mathrm{Local}})$.
Combining both inequalities yields $c \le \min(a,b)$.
Therefore, the optimum is smaller than $a$ and $b$, i.e., $g(\lambda_{\mathrm{FedIT}}^{*}) \leq \min\{a,b\}$.

{\color{black}
When $a+b=2c$, the quadratic objective becomes $g(\lambda_{\mathrm{FedIT}}) =(a-b)\lambda_{\mathrm{FedIT}}+b$, so the optimum is attained at one of the two endpoints. This degenerate case is rare, since it requires a sufficiently large correlation term, which is unlikely under task heterogeneity.
}

So far, we have shown that, at the optimal mixing weights $\lambda^{*}_{\mathrm{FedIT}}$ and $\lambda^{*}_{\mathrm{Local}}$, $\mathrm{tr}(\bm{\Sigma}_{\mathrm{Merge}})$ is smaller than $\min\{\mathrm{tr}(\bm{\Sigma}_{\mathrm{FedIT}}), \mathrm{tr}(\bm{\Sigma}_{\mathrm{FedIT}})\}$.
Therefore, the merged model admits an excess loss upper bound that is tighter than the bounds for FedIT or local
fine-tuning alone.

\subsection{Proof of Lemmas}

\begin{lemma}\label{identity}
    Let \(\bm{\mu}=\mathbb{E}[\bm{\phi}]\) and \(\bm{\Sigma}=\mathrm{Cov}(\bm{\phi})=\mathbb{E}\big[(\bm{\phi}-\bm{\mu})(\bm{\phi}-\bm{\mu})^\top\big]\).
    We have
    \begin{align}
    \mathbb{E}[\bm{\phi}^\top \mathbf{H} \bm{\phi}]
    =(\mathbb{E}\bm{\phi})^\top \mathbf{H}(\mathbb{E}\bm{\phi})+\mathrm{tr}(\mathbf{H}\,\mathrm{Cov}(\bm{\phi}))
    \end{align}
\end{lemma}

\begin{proof}
\begin{align}
    \mathbb{E}[\bm{\phi}^\top \mathbf{H} \bm{\phi}]
    &= \mathbb{E}\big[\mathrm{tr}(\bm{\phi}^\top \mathbf{H} \bm{\phi})\big]
    = \mathbb{E}\big[\mathrm{tr}(\mathbf{H}\bm{\phi}\bm{\phi}^\top)\big]
    = \mathrm{tr}\big(\mathbf{H}\,\mathbb{E}[\bm{\phi}\bm{\phi}^\top]\big) \notag
    \\
    &= \mathrm{tr}\big(\mathbf{H}(\bm{\Sigma}+\bm{\mu}\bm{\mu}^\top)\big)
    = \mathrm{tr}(\mathbf{H}\bm{\Sigma})+\mathrm{tr}(\mathbf{H}\bm{\mu}\bm{\mu}^\top) \notag
    \\&= \mathrm{tr}(\mathbf{H}\bm{\Sigma})+\bm{\mu}^\top \mathbf{H}\bm{\mu}
\end{align}
where we used \(\mathbb{E}[\bm{\phi}\bm{\phi}^\top]=\bm{\Sigma}+\bm{\mu}\bm{\mu}^\top\) and
\(\mathrm{tr}(\mathbf{AB})=\mathrm{tr}(\mathbf{BA})\).
\end{proof}

\begin{lemma}\label{lemma2}
Assume \( \begin{bmatrix} \bm{\Sigma}_1 & \bm{\Sigma}_{12}\\ \bm{\Sigma}_{12}^\top & \bm{\Sigma}_2 \end{bmatrix}\succeq 0. \) Then
\begin{align}
    \big|\mathrm{tr}(\bm{\Sigma}_{12})\big| \;\le\; \sqrt{\mathrm{tr}(\bm{\Sigma}_1)\,\mathrm{tr}(\bm{\Sigma}_2)}
\end{align}
\end{lemma}
\begin{proof}
Since
\(
\begin{bmatrix}
\bm{\Sigma}_1 & \bm{\Sigma}_{12}\\
\bm{\Sigma}_{12}^\top & \bm{\Sigma}_2
\end{bmatrix}\succeq 0,
\)
there exists a matrix $\mathbf R$ with $\|\mathbf R\|_2\le 1$ such that
\begin{align}
    \bm{\Sigma}_{12}=\bm{\Sigma}_1^{1/2}\mathbf R\,\bm{\Sigma}_2^{1/2}.
\end{align}
By applying trace,
\begin{align}
    \mathrm{tr}(\bm{\Sigma}_{12})
= \mathrm{tr}\!\big(\bm{\Sigma}_1^{1/2}\mathbf R\,\bm{\Sigma}_2^{1/2}\big).
\end{align}
Using the trace inequality $|\mathrm{tr}(\mathbf{A}^\top \mathbf{B})|\le \|\mathbf{A}\|_F\|\mathbf{B}\|_F$ and
$\|\mathbf R X\|_F\le \|\mathbf R\|_2\|X\|_F$, we have
\begin{align}
\big|\mathrm{tr}(\bm{\Sigma}_{12})\big|
&= \big|\mathrm{tr}\!\big((\bm{\Sigma}_1^{1/2})^\top(\mathbf R\bm{\Sigma}_2^{1/2})\big)\big|\notag\\
&\le \|\bm{\Sigma}_1^{1/2}\|_F\,\|\mathbf R\bm{\Sigma}_2^{1/2}\|_F\notag\\
&\le \|\bm{\Sigma}_1^{1/2}\|_F\,\|\mathbf R\|_2\,\|\bm{\Sigma}_2^{1/2}\|_F\notag\\
&\le \|\bm{\Sigma}_1^{1/2}\|_F\,\|\bm{\Sigma}_2^{1/2}\|_F.
\end{align}
Note that $
\|\bm{\Sigma}_1^{1/2}\|_F^2
= \mathrm{tr}((\bm{\Sigma}_1^{1/2})^\top \bm{\Sigma}_1^{1/2})
= \mathrm{tr}(\bm{\Sigma}_1^{1/2}\bm{\Sigma}_1^{1/2})
= \mathrm{tr}(\bm{\Sigma}_1).
$
Similarly, 
\(
\|\bm{\Sigma}_2^{1/2}\|_F^2=\mathrm{tr}(\bm{\Sigma}_2)
\).
Therefore,
\[
\big|\mathrm{tr}(\bm{\Sigma}_{12})\big|
\le \sqrt{\mathrm{tr}(\bm{\Sigma}_1)\,\mathrm{tr}(\bm{\Sigma}_2)}.
\]
\end{proof}

\newpage
\section{Merging Procedure}\label{appendix:estimation_procedure}

\noindent Step 1: Fisher Information Matrix Estimation.

We compute the diagonal Fisher information matrices $\hat{\mathbf{F}}_{\bm{\mu}_\mathrm{FedIT}}$ and $\hat{\mathbf{F}}_{\bm{\mu}_\mathrm{Local}}$ using \eqref{eq:diag_fisher}, and let $\bm{\Sigma}_{\mathrm{FedIT}} \approx \hat{\mathbf{F}}_{\bm{\mu}_\mathrm{FedIT}}^{-1}$ and $\bm{\Sigma}_{\mathrm{Local}} \approx \hat{\mathbf{F}}_{\bm{\mu}_\mathrm{Local}}^{-1}$.

\noindent Step 2: Covariance Estimation. 
Let $\{(\mathbf{x}_n,\mathbf{y}_n)\}_{n=1}^N$ be a random data batch from training set of client $i$.
For two well-trained models with parameters $\bm{\mu}_{\mathrm{FedIT}}$ and $\bm{\mu}_{\mathrm{Local}}$, the per-sample gradients are given by
\begin{align}
    \mathbf{g}_{\mathrm{FedIT}}^{(n)} \,\triangleq\, \nabla_{\bm{\mu}_{\mathrm{FedIT}}}\,\ell(\bm{\mu}_{\mathrm{FedIT}}; \mathbf{x}_n,\mathbf{y}_n),
\qquad
\mathbf{g}_{\mathrm{Local}}^{(n)} \,\triangleq\, \nabla_{\bm{\mu}_{\mathrm{Local}}}\,\ell(\bm{\mu}_{\mathrm{Local}}; \mathbf{x}_n,\mathbf{y}_n),
\end{align}
where $\ell$ is the per-sample task loss. For each parameter coordinate $k\in\{1,\dots,d\}$ where $d$ denotes the number of trainable parameters, compute the mean and variance of per-sample gradient, and covariance as
\begin{align}
    \bar g_{{\mathrm{FedIT}},k} \,=\, \frac{1}{N}\sum_{n=1}^N g_{{\mathrm{FedIT}},k}^{(n)},\quad
    \bar g_{{\mathrm{Local}},k} \,=\, \frac{1}{N}\sum_{n=1}^N g_{{\mathrm{Local}},k}^{(n)},
\end{align}
\begin{align}
    \mathrm{Var}_{\mathrm{FedIT}}(k) \,=\, \frac{1}{N}\sum_{n=1}^N \!\big(g_{{\mathrm{FedIT}},k}^{(n)}-\bar g_{{\mathrm{FedIT}},k}\big)^{2},
\quad
\mathrm{Var}_{\mathrm{Local}}(k) \,=\, \frac{1}{N}\sum_{n=1}^N \!\big(g_{{\mathrm{Local}},k}^{(n)}-\bar g_{{\mathrm{Local}},k}\big)^{2},
\end{align}
\begin{align}
    \mathrm{Cov}_{\mathrm{Cross}}(k) \,=\, \frac{1}{N}\sum_{n=1}^N \!\big(g_{{\mathrm{FedIT}},k}^{(n)}-\bar g_{{\mathrm{FedIT}},k}\big)\big(g_{{\mathrm{Local}},k}^{(n)}-\bar g_{{\mathrm{Local}},k}\big).
\end{align}
The gradient correlation coefficient for coordinate $k$ is defined as
\begin{align}
    \rho_k \;\triangleq\;
\frac{\mathrm{Cov}_{\mathrm{Cross}}(k)}{\sqrt{\mathrm{Var}_{\mathrm{FedIT}}(k)\,\mathrm{Var}_{\mathrm{Local}}(k)}}\,,
\qquad \rho_k\in[-1,1].
\end{align}
We adopt a diagonal approximation for the cross-covariance,
\begin{align}
    (\bm{\Sigma}_{\mathrm{Cross}})_{kk} \;\triangleq\; \rho_k\,\sqrt{(\bm{\Sigma}_{\mathrm{FedIT}})_{kk}\,(\bm{\Sigma}_{\mathrm{Local}})_{kk}},
\qquad k\in\{1,\dots,d\}.
\end{align}
Then, we clip $\rho_k$ to ensure $\bm{\Sigma}_{\mathrm{Cross}}\succeq 0$ and $\bm{\Sigma}_{\mathrm{Cross}}\preceq \bm{\Sigma}_{\mathrm{FedIT}},\bm{\Sigma}_{\mathrm{Local}}$:
\begin{align}
&  \rho_k \triangleq \max\!\bigl\{0,\ \min\{\rho_k,\ \rho_k^{\max}\}\bigr\},
\quad \mathrm{where}\,\, 
\rho_k^{\max}
= \min\!\left\{\sqrt{\frac{(\bm{\Sigma}_{\mathrm{FedIT}})_{kk}}{(\bm{\Sigma}_{\mathrm{Local}})_{kk}}},\,
\sqrt{\frac{(\bm{\Sigma}_{\mathrm{Local}})_{kk}}{(\bm{\Sigma}_{\mathrm{FedIT}})_{kk}}}\right\}.
\end{align}
This yields a PSD diagonal $\bm{\Sigma}_{\mathrm{Cross}}$ with each element bounded by the corresponding entries of $\bm{\Sigma}_{\mathrm{FedIT}}$ and $\bm{\Sigma}_{\mathrm{Local}}$.

Step 3: Mixing Weights Calculation. 
The scalars $a$, $b$, and $c$ are calculated as 
\begin{align}
    &a = \mathrm{tr}(\bm{\Sigma}_{\mathrm{FedIT}}), 
    \,
    b = \mathrm{tr}(\bm{\Sigma}_{\mathrm{Local}}), 
    \,
    \\&c \;=\; \mathrm{tr}(\bm{\Sigma}_{\mathrm{Cross}})
\,=\, \sum_{k=1}^d \,(\bm{\Sigma}_{\mathrm{Cross}})_{kk}
\,=\, \sum_{k=1}^d \,\rho_k\,\sqrt{(\bm{\Sigma}_{\mathrm{FedIT}})_{kk}\,(\bm{\Sigma}_{\mathrm{Local}})_{kk}}.
\end{align}

\newpage
\section{Uplink Communication Cost Comparison}
We denote $T$ as the number of communication rounds.
FedIT uploads both $\mathbf{A}$ and $\mathbf{B}$ each round, yielding per-client cost $\mathcal{O}(r(m+n)T)$.
FedSA uploads only $\mathbf{A}$, yielding $\mathcal{O}(rnT)$.
FFA-LoRA uploads only $\mathbf{B}$, yielding $\mathcal{O}(rmT)$.
{\color{black}The communication cost of FedDPA, FedALT, PerAda, and APFL are $\mathcal{O}(r(m+n)T)$ because all of them uploads the local LoRA module each round.
FedMosaic requires each client to upload the local last-layer gradients for computing client similarity coefficients. Since this gradient aggregation is performed periodically, its communication cost is negligible compared with the cost of transmitting LoRA modules. As a result, the overall communication cost of FedMosaic is $\mathcal{O}(r(m+n)T)$.
The uplink communication cost of L2GD is $\mathcal{O}(rp(m+n)(1-p)T)$, where $p$ denotes the probability of performing an aggregation step, since uplink transmission occurs only when the coin toss switches from a local update step to an aggregation step.}
For \textsc{Potara}, clients upload both $\mathbf{A}$ and $\mathbf{B}$ only during the FedIT phase, so the cost is $\mathcal{O}(r(m+n)T_{\text{FedIT}})$, where $T_{\text{FedIT}}$ denotes
the number of communication rounds required to train the FedIT model.
The communication cost comparison is summarized in Table \ref{comm_comparison}.
{\textsc{Potara}} achieves consistent performance gains even under a reduced communication budget.
A detailed performance-communication trade-off is reported in Section \ref{acc_com_trade_off}.

\begin{table*}[h]
\caption{Communication Cost Comparison.}
\label{comm_comparison}
\centering
\small
\resizebox{0.45\linewidth}{!}{
\begin{tabular}{l|c}
\multicolumn{2}{c}{} \\
\toprule
\textbf{Methods} 
& \textbf{Communication Overhead}  \\
\midrule
FedIT 
 & $\mathcal{O}(r(m+n)T)$\\
FedSA 
 & $\mathcal{O}(rnT)$ \\
FFA-LoRA 
& $\mathcal{O}(rmT)$\\
FedALT 
& $\mathcal{O}(r(m+n)T)$ \\
FedDPA 
& $\mathcal{O}(r(m+n)T)$ \\
L2GD
& $\mathcal{O}(rp(m+n)(1-p)T)$ \\
FedMosaic
& $\mathcal{O}(r(m+n)T)$ \\
APFL
& $\mathcal{O}(r(m+n)T)$ \\
PerAda
& $\mathcal{O}(r(m+n)T)$ \\
\rowcolor{ours} \textsc{Potara}
& $\mathcal{O}(r(m+n)T_{\text{FedIT}})$\\
\bottomrule
\end{tabular}
}
\end{table*}

\section{Computation Cost Comparison}\label{appendix:comnputation_comparison}

We compare the proposed method with baselines in terms of computation time and accuracy. 
The results are reported in Table~\ref{tab:time_comparison}. 
All the experiments are conducted on NVIDIA A100 GPU with 40 GB of memory.
FedDPA, FedMosaic, APFL, and PerAda are dual-adapter based methods. Among them, FedDPA, APFL, and PerAda update both the global adapter and the local adapter during training, which substantially increases the training cost. As a result, these methods require nearly twice the computation time of FedIT. 
Although FedMosaic mainly updates the local adapter with frozen global adapter, it is still more time-consuming than FedIT because it additionally computes gradient-based similarity coefficients and updates the gating parameter during training.
FedALT also needs to update the local adapter and the mixer weight while keeping the global adapter frozen, which incurs additional cost than FedIT.
In contrast, \textsc{Potara} achieves the lowest total computation time among all compared methods and requires only $3.91$ hours in total. This is significantly lower than FedIT, which requires $5.92$ hours, and much lower than FedDPA, APFL, and PerAda, whose total computation times are around $11$ hours. The additional estimation steps in \textsc{Potara} introduce only a small computational overhead, while avoiding expensive dual-adapter federated training.
More importantly, the computational efficiency of \textsc{Potara} does not come at the cost of performance. \textsc{Potara} achieves the highest accuracy of $74.18$, outperforming baselines. These results demonstrate that \textsc{Potara} provides a more favorable efficiency-accuracy trade-off than baselines.
Overall, the results show that the proposed post-training merging strategy is computationally efficient and effective.

We further compares \textsc{Potara} with FedIT under the same total training cost in Table~\ref{tab:same_training_cost}.
On CIFAR-100, \textsc{Potara} uses the FedIT checkpoint trained for $300$ communication rounds and merges it with the local model trained for $300$ rounds, resulting in a total training cost of $600$ rounds.
Under the same training cost, FedIT trained alone for $600$ rounds achieves $61.28$ accuracy, whereas \textsc{Potara} achieves a higher accuracy of $62.98$ while requiring only $300$ communication rounds.
A similar trend is observed on the commonsense reasoning benchmark.
\textsc{Potara} can terminate federated training earlier at $100$ rounds and merge the resulting FedIT model with the local model trained for $300$ rounds, yielding a total training cost of $400$ rounds and an accuracy of $74.10$.
In contrast, FedIT trained alone for $400$ rounds achieves only $71.60$ accuracy.
These results show that, under the same total training cost, \textsc{Potara} achieves better performance than FedIT while substantially reducing communication overhead.

\begin{table}[h]
\centering
\small
\caption{Comparison between \textsc{Potara} and FedIT under the same total training cost. For \textsc{Potara}, the total training cost includes federated training and local training, while communication cost only comes from federated training.}
\label{tab:same_training_cost}
\resizebox{\linewidth}{!}{
\begin{tabular}{llccc}
\toprule
\textbf{Benchmark}
& \textbf{Method}
& \textbf{Training Cost}
& \textbf{Communication Cost}
& \textbf{Accuracy} \\
\midrule
\multirow{2}{*}{CIFAR-100}
& FedIT
& $600$ rounds
& $600$ rounds
& $61.28$ \\
& \cellcolor{ours}\textsc{Potara}
& \cellcolor{ours}$300$ FedIT $+$ $300$ local
& \cellcolor{ours}$300$ rounds
& \cellcolor{ours}$\mathbf{62.98}$ \\
\midrule
\multirow{2}{*}{Commonsense Reasoning}
& FedIT
& $400$ rounds
& $400$ rounds
& $71.60$ \\
& \cellcolor{ours}\textsc{Potara}
& \cellcolor{ours}$100$ FedIT $+$ $300$ local
& \cellcolor{ours}$100$ rounds
& \cellcolor{ours}$\mathbf{74.10}$ \\
\bottomrule
\end{tabular}
}
\end{table}

\section{Experiments on FLAN benchmark}\label{appendix:larger_model}

\subsection{Performance Comparison}

We evaluate \textsc{Potara} on the FLAN benchmark~\citep{chung2024scaling}, which consists of over 60 datasets covering more than 12 NLP task categories.
To simulate an extreme heterogeneity scenario, we consider eight clients, each trained on a completely different dataset. 
This setting is highly task-heterogeneous and is substantially more extreme than the CIFAR-100 and commonsense reasoning setups.
Dataset details, task descriptions, and client assignments are summarized in Appendix~\ref{FLAN_data}. 
We adopt Qwen3-4B-Instruct-2507~\citep{yang2025qwen3} as the pre-trained model and report ROUGE-1 for all tasks. 
Baselines are trained for 600 rounds, except local fine-tuning which runs for 300 rounds to mitigate overfitting. 
For \textsc{Potara}, we train FedIT for 300 rounds and local fine-tuning for 300 rounds, and merge the two resulting models to obtain the final model.
As shown in Table \ref{table_flan}, \textsc{Potara} achieves the best overall ROUGE-1 on the FLAN benchmark and consistently performing strongly across diverse task categories, indicating that merging global knowledge with personalized knowledge is especially effective under extreme task heterogeneity.
Meanwhile, \textsc{Potara} attains these improvements while reducing communication cost because it merges an intermediate FedIT model with local model.

\begin{table*}[h]
\caption{ROUGE-1 performance on FLAN benchmark (Qwen3-4B-Instruct model).}
\label{table_flan}
\centering
\small
\resizebox{\linewidth}{!}{
\begin{tabular}{l|cccccccc|c}
\multicolumn{10}{c}{\textbf{FLAN benchmark (Qwen3-4B-Instruct model)}} \\
\toprule
\textbf{Methods} 
& \textbf{Closed-book QA} 
& \textbf{Reading Comp.} 
& \textbf{Coreference} 
& \textbf{Paraphrase} 
& \textbf{Commonsense} 
& \textbf{NLI} 
& \textbf{Reading} 
& \textbf{Sentiment} 
& \textbf{Avg.} \\
\midrule
Local       
& 69.53 {\com 0.75} & 69.38 {\com 2.81} & 88.01 {\com 0.70} & 80.00 {\com 3.56} & 87.62 {\com 0.78} & 85.78 {\com 2.18} & 73.87 {\com 0.66} & 76.16 {\com 1.69} & 78.79 {\com 0.99} \\
FFA-LoRA    
& 71.00 {\com 0.66} & 70.54 {\com 0.66} & 77.18 {\com 1.34} & 79.89 {\com 0.57} & 85.65 {\com 1.18} & 86.33 {\com 1.18} & 73.57 {\com 0.39} & 77.06 {\com 0.74} & 77.65 {\com 0.70} \\
FedSA       
& 69.54 {\com 0.95} & 69.54 {\com 1.09} & 87.78 {\com 2.61} & 81.89 {\com 0.87} & 84.05 {\com 1.16} & 86.00 {\com 1.19} & 72.72 {\com 0.65} & 76.47 {\com 0.98} & 78.50 {\com 0.25} \\
FedIT       
& 69.96 {\com 0.57} & 71.37 {\com 0.63} & 83.19 {\com 0.51} & 80.44 {\com 1.40} & 85.37 {\com 0.96} & 86.33 {\com 1.25} & 73.78 {\com 0.85} & 76.68 {\com 0.56} & 78.39 {\com 0.63} \\
FedDPA       
& 70.03 {\com 0.19} & 70.57 {\com 0.34} & 81.89 {\com 1.37} & 81.44 {\com 0.79} & 86.93 {\com 1.71} & 87.55 {\com 0.88} & 74.01 {\com 0.02} & 76.45 {\com 0.88} & 78.61 {\com 0.29} \\
FedALT& 68.07 {\com 1.51} & 72.50 {\com 0.77} & 84.84 {\com 0.51} & 82.44 {\com 1.03} & 85.17 {\com 0.07} & 86.78 {\com 0.57} & 73.43 {\com 0.86} & 77.40 {\com 1.02} & 78.83 {\com 0.26} \\
L2GD &70.45{\com0.48  } & 69.80{\com1.63  } & 83.24{\com1.73  } & 79.83{\com1.50  } &84.29{\com1.58  } & 86.34{\com1.34 } & 72.13{\com0.62 } &76.39{\com0.12 } & 77.81{\com0.42  }\\
FedMosaic  &69.33{\com0.94 } & 70.39{\com1.18 } & 77.79{\com2.67 } & 81.00{\com2.00 } & 85.56{\com2.81 } & 82.83{\com2.50 } & 72.90{\com1.14 } & 76.20{\com0.01 } & 77.01{\com1.65 } \\
APFL  &69.55{\com1.02 } & 69.79{\com2.55 } & 82.44{\com2.18 } & 77.84{\com0.84 } & 84.74{\com1.65 } & 85.50{\com0.83 } & 73.17{\com0.82} & 75.86{\com0.30 } & 77.37{\com0.45 }\\
PerAda &70.08{\com0.69  } &70.28{\com0.93  } & 87.77{\com1.14  } & 80.67{\com1.00  } &87.00{\com2.06  } & 85.00{\com1.67  } & 73.24{\com1.03  } &75.85{\com0.05  } & 78.74{\com0.14 } \\
\rowcolor{ours} \textsc{Potara}     
& 70.61 {\com 0.54} & 70.61 {\com 2.06} & 88.26 {\com 1.77} & 81.56 {\com 0.42} & 88.80 {\com 0.23} & 86.78 {\com 1.10} & 75.18 {\com 0.78} & 76.10 {\com 1.01} & 79.74 {\com 0.27} \\
\bottomrule
\end{tabular}
}
\end{table*}

\subsection{FLAN Benchmark} \label{FLAN_data}

We follow the experimental setup in \citep{bian2025fedalt,yang2024dual} to simulate a client-level task heterogeneity setting using the FLAN datasets \citep{chung2024scaling}.
The FLAN collection comprises over 12 natural language processing tasks, each of which contains multiple datasets.
These tasks are generated from diverse contextual factors and inherently exhibit complex and heterogeneous data distribution shifts.
We consider an extreme client heterogeneity setting where 8 clients are assigned completely different NLP tasks. 
To simulate a realistic environment where each client has limited access to large-scale datasets, we adopt a downsampling strategy that assigns 2000 training samples and 300 test samples to each client.
The specific task and dataset assignments are summarized in Table~\ref{FLAN_data_table}.

\begin{table}[h]
\caption{FLAN benchmark.}
\label{FLAN_data_table}
\centering
\resizebox{\linewidth}{!}{
\begin{tabular}{l l l}
\toprule
\textbf{Task} & \textbf{Dataset} & \textbf{Description} \\
\midrule
Closed-book QA 
& arc\_challenge 
& Answer multiple-choice questions. \\

Reading Comp. w/ Commonsense 
& cosmos\_qa 
& Reading comprehension with commonsense. \\

Coreference Resolution 
& definite\_pronoun\_resolution 
& Resolve ambiguous pronouns. \\

Paraphrase Detection 
& glue\_qqp 
& Identify semantic equivalence. \\

Commonsense Reasoning 
& hellaswag 
& Choose the most plausible continuation. \\

Natural Language Inference 
& mnli 
& Infer relations between sentence pairs. \\

Reading Comprehension 
& squad\_v1 
& Extract answers from passages. \\

Sentiment Analysis 
& sst2 
& Sentence-level sentiment classification. \\
\bottomrule
\end{tabular}
}
\end{table}

\section{Experiments with Different Ranks}\label{appendix:different_ranks}
{\color{black}
We evaluate the proposed method under different LoRA ranks and report the results on personalized CIFAR-100 and commonsense reasoning benchmarks in Tables~\ref{tab:cifar_rank} and~\ref{tab:commonsense_rank}, respectively.
As shown in Table~\ref{tab:cifar_rank}, increasing the LoRA rank generally improves the performance of all methods. Note that our method consistently achieves the best performance across all ranks. In particular, on personalized CIFAR-100, our method obtains accuracies of $58.26$, $65.21$, and $72.66$ for ranks $4$, $8$, and $16$, respectively, outperforming the strongest baseline at each rank. 
Table~\ref{tab:commonsense_rank} shows the results on the commonsense reasoning benchmark. Our method again achieves the highest average accuracy for all ranks, with accuracies of $74.18$, $73.68$, and $75.55$ for ranks $8$, $16$, and $32$, respectively. These results indicate that the proposed method is not sensitive to a specific rank choice and remains effective under different parameter-efficient adaptation budgets.
Overall, the consistent improvements across both vision and language benchmarks demonstrate the adaptability and robustness of our method. 
}

\begin{table*}[h]
\caption{Performance comparison on Personalized CIFAR-100 benchmark with different ranks.}
\label{tab:cifar_rank}
\centering
\small
\resizebox{\linewidth}{!}{
\begin{tabular}{l|cccccccccc>{\columncolor{ours}}c}
\multicolumn{11}{c}{\textbf{Personalized CIFAR-100 benchmark}} \\
\toprule
\textbf{Rank} 
&\textbf{Local} 
& \textbf{FFA-LoRA} 
& \textbf{FedSA} 
& \textbf{FedIT} 
& \textbf{FedDPA} 
& \textbf{Fisher Merging} 
& \textbf{FedMosaic} 
& \textbf{PerAda} 
& \textbf{L2GD} 
& \textbf{APFL} 
& \textbf{Potara} 
\\
\midrule
4 & 55.24 &  33.77 &  55.22 &  55.28 & 55.26 &  55.19 &  55.22 &  54.34 &  51.14 &  57.64 & 58.26  \\
8 & 60.43 & 41.52 & 60.87 & 61.28 & 61.80 & 61.72 & 61.53 & 60.64 & 53.43 & 61.91 & 65.21 \\
16 & 63.05 & 51.62 & 66.67 & 68.27 & 68.66 & 69.95 & 67.75 & 67.83 & 66.51 & 69.96 & 72.66 \\
\bottomrule
\end{tabular}
}
\end{table*}

\begin{table*}[h]
\caption{Performance comparison on commonsense reasoning benchmark with different ranks.}
\label{tab:commonsense_rank}
\centering
\small
\resizebox{\linewidth}{!}{
\begin{tabular}{l|ccccccccc>{\columncolor{ours}}c}
\multicolumn{11}{c}{\textbf{Commonsense reasoning benchmark}} \\
\toprule
\textbf{Rank} 
&\textbf{Local} 
& \textbf{FFA-LoRA} 
& \textbf{FedSA} 
& \textbf{FedIT} 
& \textbf{FedDPA} 
& \textbf{FedMosaic} 
& \textbf{PerAda} 
& \textbf{L2GD} 
& \textbf{APFL} 
& \textbf{Potara} 
\\
\midrule
8 & 72.05 & 68.80 & 71.15 & 71.94 & 72.33 & 71.14 & 71.95 & 72.63 & 72.30 & 74.18  \\
16 & 72.20 & 70.75 & 72.00 & 72.35 & 72.25 & 71.85 & 72.47 & 72.13 & 72.00 & 73.68 \\
32 & 72.33 & 71.12 & 72.78 & 72.27 & 72.35 & 71.63 & 72.10 & 72.63 & 71.67 & 75.55 \\
\bottomrule
\end{tabular}
}
\end{table*}

\newpage
\section{Experiments with Increased Number of Clients}\label{appendix:more_clients}
In this section, we evaluate the scalability of \textsc{Potara} by increasing the number of clients across different benchmarks.
For personalized CIFAR-100, we first increase the number of clients to $20$, with each client following a distinct label distribution.
For commonsense reasoning, we increase the number of clients to $16$, where every two clients share the same task but are assigned different training samples.
As shown in Tables~\ref{table_cifar100_20clients} and~\ref{table_cs_16clients}, \textsc{Potara} consistently outperforms all baselines under these larger-client settings, demonstrating robust performance gains as the federation scale increases.
To further examine scalability, we increase the number of clients to $50$ on personalized CIFAR-100.
The results in Table~\ref{table_cifar100_50clients} show that \textsc{Potara} continues to maintain its advantage in average performance, further validating its effectiveness in larger-scale federated settings.

\begin{table*}[h]
\caption{Performance comparison on personalized CIFAR-100 benchmark with 20 clients.}
\label{table_cifar100_20clients}
\centering
\small
\resizebox{0.9\linewidth}{!}{%
\begin{tabular}{l|cccccccccc >{\columncolor{ours}}c}
\toprule
Client & Local & FFA-LoRA & FedSA & FedIT & FedDPA & Fisher Merging  & FedMosaic & L2GD & APFL & PerAda & \textsc{Potara}\\
\midrule
Client 1  & 63.53 {\com 2.22} & 44.37 {\com 1.97} & 61.28 {\com 2.27} & 59.52 {\com 0.73} & 67.41 {\com 2.92} & 62.00 {\com 2.46} & 64.62 {\com 2.60} & 60.97 {\com 4.96} &62.24{\com0.58 }  &57.95 {\com 2.30} & 62.39 {\com 3.91} \\
Client 2  & 62.28 {\com 5.39} & 49.25 {\com 6.36} & 61.16 {\com 5.72} & 61.35 {\com 6.12} & 65.95 {\com 6.74} & 61.39 {\com 9.09} & 65.04 {\com 5.84} & 63.22 {\com 4.65} &64.62{\com6.26 } &63.70 {\com 6.38} & 70.18 {\com 5.12} \\
Client 3  & 60.40 {\com 0.23} & 45.24 {\com 1.43} & 58.25 {\com 2.07} & 64.25 {\com 2.42} & 65.39 {\com 3.34} & 59.64 {\com 2.42} & 65.44 {\com 1.71} & 64.39 {\com 2.12} &64.33{\com0.60 } &63.91 {\com 2.28} & 70.09 {\com 1.34} \\
Client 4  & 58.45 {\com 2.62} & 39.74 {\com 0.34} & 57.29 {\com 1.19} & 56.98 {\com 2.56} & 59.52 {\com 0.55} & 62.54 {\com 0.71} & 57.63 {\com 1.91} & 59.14 {\com 0.93} &58.55{\com1.69 } &60.39 {\com 1.43} & 63.63 {\com 1.26} \\
Client 5  & 62.07 {\com 2.18} & 50.88 {\com 4.72} & 62.73 {\com 2.31} & 66.00 {\com 3.74} & 68.64 {\com 1.89} & 66.22 {\com 0.01} & 66.05 {\com 1.93} & 65.06 {\com 2.53} &65.22{\com0.59 } &61.74 {\com 2.11} & 67.94 {\com 4.36} \\
Client 6  & 64.14 {\com 0.67} & 47.09 {\com 2.28} & 61.07 {\com 0.85} & 65.06 {\com 2.94} & 65.34 {\com 2.65} & 59.56 {\com 0.25} & 64.63 {\com 1.55} & 63.00 {\com 3.19} &62.04{\com0.70 } &58.93 {\com 3.96} & 67.00 {\com 0.09} \\
Client 7  & 58.82 {\com 0.95} & 40.24 {\com 1.49} & 58.12 {\com 0.59} & 57.96 {\com 1.45} & 58.85 {\com 3.59} & 61.12 {\com 2.24} & 58.42 {\com 5.12} & 57.97 {\com 0.07} &59.31{\com5.34 }  &60.10 {\com 0.13} & 61.66 {\com 0.95} \\
Client 8  & 62.78 {\com 1.61} & 44.65 {\com 0.02} & 59.21 {\com 0.43} & 63.37 {\com 3.22} & 64.41 {\com 1.21} & 63.31 {\com 0.11} & 61.06 {\com 0.65} & 62.07 {\com 0.12} &64.75{\com2.06 } &61.29 {\com 1.14} & 65.35 {\com 2.42} \\
Client 9  & 62.58 {\com 1.40} & 42.76 {\com 2.31} & 60.06 {\com 1.32} & 59.63 {\com 4.14} & 60.85 {\com 3.94} & 65.22 {\com 0.23} & 59.65 {\com 0.71} & 56.71 {\com 2.44} &64.21{\com1.44 } &58.23 {\com 1.12} & 66.83 {\com 0.98} \\
Client 10 & 62.14 {\com 1.33} & 44.13 {\com 2.60} & 59.58 {\com 4.50} & 62.31 {\com 2.31} & 65.35 {\com 1.79} & 62.25 {\com 0.61} & 61.45 {\com 4.46} & 59.74 {\com 1.48} &60.25{\com1.79 } &62.54 {\com 1.54} & 66.92 {\com 1.45} \\
Client 11 & 60.16 {\com 1.79} & 44.74 {\com 1.14} & 59.23 {\com 2.08} & 64.24 {\com 1.03} & 62.61 {\com 1.03} & 57.62 {\com 0.73} & 60.63 {\com 1.44} & 60.34 {\com 0.98} &61.86{\com2.67 } &57.31 {\com 3.37} & 61.96 {\com 0.85} \\
Client 12 & 57.86 {\com 1.87} & 44.80 {\com 2.32} & 55.35 {\com 0.39} & 65.66 {\com 1.49} & 60.01 {\com 0.94} & 59.58 {\com 2.15} & 63.15 {\com 2.97} & 61.04 {\com 1.97} &59.29{\com3.10 }  &58.56 {\com 2.80} & 63.40 {\com 1.89} \\
Client 13 & 61.04 {\com 1.68} & 46.04 {\com 7.54} & 61.00 {\com 1.22} & 61.74 {\com 1.95} & 64.73 {\com 3.67} & 60.20 {\com 4.24} & 63.10 {\com 4.80} & 60.01 {\com 1.72} &63.73{\com2.67 } &63.53 {\com 0.08} & 66.48 {\com 1.17} \\
Client 14 & 60.64 {\com 2.41} & 43.08 {\com 1.47} & 58.89 {\com 4.14} & 60.73 {\com 3.29} & 59.39 {\com 8.12} & 64.56 {\com 4.91} & 58.33 {\com 5.48} & 58.76 {\com 1.16} &61.27{\com6.05 } &61.54 {\com 1.89} & 63.97 {\com 4.32} \\
Client 15 & 59.61 {\com 1.47} & 37.66 {\com 0.69} & 57.57 {\com 1.00} & 61.24 {\com 1.12} & 60.62 {\com 2.62} & 59.30 {\com 1.39} & 57.44 {\com 0.02} & 56.45 {\com 0.77} &57.81{\com2.32 } &57.97 {\com 1.18} & 63.62 {\com 2.54} \\
Client 16 & 57.85 {\com 0.45} & 42.66 {\com 0.24} & 57.34 {\com 0.59} & 60.27 {\com 3.41} & 58.12 {\com 0.53} & 59.75 {\com 2.71} & 61.15 {\com 3.39} & 58.37 {\com 1.15} &59.74{\com2.71 } &59.44 {\com 4.75} & 63.39 {\com 1.66} \\
Client 17 & 62.09 {\com 4.21} & 43.70 {\com 8.42} & 58.79 {\com 2.54} & 64.45 {\com 5.90} & 61.59 {\com 4.52} & 63.83 {\com 3.02} & 61.40 {\com 5.63} & 58.00 {\com 7.75} &61.32{\com5.88 } &61.11 {\com 2.42} & 64.84 {\com 3.86} \\
Client 18 & 58.45 {\com 2.12} & 43.16 {\com 0.08} & 57.26 {\com 0.93} & 59.52 {\com 0.82} & 57.10 {\com 1.57} & 59.87 {\com 3.55} & 61.37 {\com 1.10} & 60.15 {\com 4.22} &58.72{\com1.61 } &58.47 {\com 3.53} & 64.88 {\com 2.82} \\
Client 19 & 62.75 {\com 3.68} & 50.26 {\com 1.38} & 63.88 {\com 1.85} & 66.43 {\com 1.42} & 66.52 {\com 2.21} & 66.39 {\com 2.62} & 69.22 {\com 0.73} & 67.32 {\com 0.42} &64.91{\com0.84 }  &62.34 {\com 1.05} & 71.42 {\com 0.54} \\
Client 20 & 58.68 {\com 0.28} & 43.53 {\com 2.69} & 60.81 {\com 0.45} & 63.67 {\com 3.51} & 62.98 {\com 1.16} & 60.39 {\com 2.02} & 63.34 {\com 0.40} & 62.07 {\com 0.72} &64.53{\com0.79 }  &59.86 {\com 1.29} & 67.26 {\com 0.28} \\
\midrule
\textbf{Average} & 60.81 {\com 0.57} & 44.40 {\com 1.09} & 59.45 {\com 0.95} & 62.22 {\com 0.93} & 62.77 {\com 0.63} & 61.73 {\com 1.45} & 62.16 {\com 1.67} & 60.74 {\com 0.27} &61.95{\com1.52 }  &60.35 {\com 0.64} & 65.66 {\com 0.82} \\
\bottomrule
\end{tabular}
}
\end{table*}

\begin{table*}[h]
\caption{Performance comparison on commonsense reasoning benchmark with 16 clients.}
\label{table_cs_16clients}
\centering
\small
\resizebox{0.9\linewidth}{!}{
\begin{tabular}{l|cccccccccc >{\columncolor{ours}}c}
\toprule
Client & Local & FFA-LoRA & FedSA & FedIT & FedALT &FedDPA & FedMosaic & L2GD & APFL & PerAda & \textsc{Potara}\\
\midrule
Client 1 (ARC-c)       &71.80 {\com 0.20} &73.00 {\com 0.80} &70.60 {\com 0.40} &73.60 {\com 0.00} &73.50 {\com 0.70} &74.50 {\com 0.90} &73.70 {\com 0.70} &75.10  {\com 0.70} &74.50{\com0.10} &71.40 {\com 1.80} &74.80 {\com 0.20} \\
Client 2 (ARC-c)       &70.50 {\com 0.50} &73.10 {\com 0.30} &72.20 {\com 1.20} &73.60 {\com 0.00} &72.40 {\com 0.00} &74.50 {\com 0.10} &74.40 {\com 0.80} &75.80  {\com 0.00} &74.10{\com0.70} &73.00 {\com 0.60} &75.20 {\com 1.00} \\
Client 3 (ARC-e)       &85.80 {\com 0.20} &87.50 {\com 0.50} &86.20 {\com 1.20} &87.30 {\com 0.70} &85.90 {\com 0.70} &86.90 {\com 0.30} &87.30 {\com 1.70} &88.00  {\com 0.20} &87.80{\com0.20} &87.40 {\com 0.60} &88.50 {\com 0.30} \\
Client 4 (ARC-e)       &86.70 {\com 0.50} &87.10 {\com 0.50} &85.90 {\com 1.30} &87.30 {\com 0.70} &85.00 {\com 0.20} &87.40 {\com 0.60} &88.20 {\com 2.00} &87.70  {\com 1.10} &88.10{\com0.50} &87.80 {\com 1.60} &87.90 {\com 0.30} \\
Client 5 (BoolQ)       &65.20 {\com 3.00} &67.00 {\com 0.60} &61.60 {\com 5.60} &64.30 {\com 3.50} &66.70 {\com 0.10} &66.90 {\com 0.10} &66.80 {\com 1.80} &66.70  {\com 0.30} &58.50{\com5.70} &64.30 {\com 0.10} &65.60 {\com 0.60} \\
Client 6 (BoolQ)       &67.90 {\com 1.50} &63.50 {\com 1.70} &64.90 {\com 0.70} &64.30 {\com 3.50} &65.50 {\com 0.30} &64.20 {\com 3.00} &67.60 {\com 0.00} &62.60  {\com 5.00} &61.70{\com3.50} &62.40 {\com 4.40} &66.70 {\com 1.70} \\
Client 7 (HellaSwag)   &68.50 {\com 1.50} &58.50 {\com 0.10} &66.70 {\com 1.90} &66.90 {\com 1.10} &69.00 {\com 0.20} &65.00 {\com 0.00} &62.10 {\com 1.10} &59.50  {\com 1.50} &59.70{\com2.30} &71.00 {\com 0.80} &69.70 {\com 1.30} \\
Client 8 (HellaSwag)   &68.10 {\com 0.10} &59.70 {\com 0.50} &63.90 {\com 2.50} &66.90 {\com 1.10} &68.00 {\com 1.20} &68.10 {\com 0.70} &67.70 {\com 1.70} &60.30  {\com 0.90} &58.70{\com3.90} &68.60 {\com 3.20} &68.30 {\com 0.90} \\
Client 9 (OBQA)        &75.90 {\com 0.50} &71.30 {\com 2.90} &72.90 {\com 0.30} &76.70 {\com 0.10} &75.80 {\com 1.00} &78.00 {\com 1.80} &76.70 {\com 0.90} &75.60  {\com 0.60} &75.40{\com0.80} &74.80 {\com 2.60} &78.70 {\com 0.50} \\
Client 10 (OBQA)       &74.90 {\com 2.50} &73.80 {\com 0.40} &73.50 {\com 1.30} &76.70 {\com 0.10} &76.00 {\com 1.00} &77.80 {\com 1.00} &75.50 {\com 1.70} &75.60  {\com 1.20} &75.00{\com0.20} &76.70 {\com 1.50} &77.90 {\com 0.70} \\
Client 11 (PIQA)       &80.20 {\com 0.80} &79.40 {\com 1.80} &77.40 {\com 0.40} &80.60 {\com 0.60} &79.80 {\com 0.20} &80.20 {\com 0.60} &80.90 {\com 0.70} &80.90  {\com 0.90} &77.60{\com2.20} &80.10 {\com0.70} &81.90 {\com 0.90} \\
Client 12 (PIQA)       &79.40 {\com 1.00} &77.60 {\com 2.60} &78.30 {\com 0.70} &80.60 {\com 0.60} &79.20 {\com 1.60} &81.20 {\com 0.60} &76.20 {\com 3.40} &80.60  {\com 0.00} &80.40{\com1.80} &77.50 {\com 1.90} &81.90 {\com 0.30} \\
Client 13 (SIQA)       &67.70 {\com 1.10} &67.60 {\com 0.20} &67.00 {\com 0.80} &70.50 {\com 0.10} &70.80 {\com 3.00} &71.30 {\com 1.50} &69.10 {\com 0.70} &65.70  {\com 4.30} &69.60{\com1.20} &69.60 {\com 0.60} &73.00 {\com 1.60} \\
Client 14 (SIQA)       &67.80 {\com 1.20} &67.70 {\com 0.70} &68.30 {\com 0.70} &70.50 {\com 0.10} &66.50 {\com 0.50} &69.20 {\com 1.80} &68.50 {\com 0.90} &64.60  {\com 5.00} &68.70{\com0.50} &69.50 {\com 0.70} &72.40 {\com 1.80} \\
Client 15 (WinoGrande) &67.00 {\com 0.20} &52.20 {\com 2.60} &66.20 {\com 0.40} &62.00 {\com 0.80} &63.40 {\com 0.00} &62.10 {\com 0.50} &60.70 {\com 1.10} &55.00  {\com 3.80} &52.50{\com0.90} &62.90 {\com 0.70} &62.80 {\com 0.60} \\
Client 16 (WinoGrande) &64.80 {\com 1.00} &51.30 {\com 1.70} &64.30 {\com 1.30} &62.00 {\com 0.80} &63.20 {\com 0.40} &61.70 {\com 0.10} &62.40 {\com 1.80} &54.40  {\com 4.80} &55.80{\com0.80}  &61.30 {\com 1.10} &63.60 {\com 0.20} \\
\midrule
\textbf{Average}  &72.64 {\com 0.34} &69.39 {\com 0.38} &71.25 {\com 0.35} &72.74 {\com 0.24} &72.55 {\com 0.53} &73.06 {\com 0.23} &72.37 {\com 0.02} &70.50  {\com 0.94} &69.88{\com0.38} &72.40 {\com 0.30} &74.31 {\com 0.58} \\
\bottomrule
\end{tabular}
}
\end{table*}

\begin{table*}[h]
\caption{Performance comparison on personalized CIFAR-100 benchmark with 50 clients.}
\label{table_cifar100_50clients}
\centering
\small
\resizebox{0.9\linewidth}{!}{%
\begin{tabular}{l|ccccccccc >{\columncolor{ours}}c}
\toprule
Method & Local & FFA-LoRA & FedSA & FedIT & FedDPA & Fisher Merging & FedMosaic & PerAda & L2GD   &\textsc{Potara}\\
\midrule
Avg Acc &57.14  &45.41 & 53.96  & 62.30 & 62.41 & 62.02 & 63.30 & 61.43 & 61.36 & 64.91\\
\bottomrule
\end{tabular}
}
\end{table*}

\section{Sensitivity under Different Mini-batch Sizes}\label{appendix:batchsize}

We conduct a sensitivity analysis by varying the mini-batch size used to estimate the diagonal Fisher information and cross-covariance. As shown in the table, the performance remains stable across different mini-batch sizes, indicating that \textsc{Potara} is robust to this estimation parameter.

\begin{table*}[h]
\caption{Sensitivity under different mini-batch sizes.}
\centering
\small
\resizebox{0.5\linewidth}{!}{%
\begin{tabular}{l|ccccccccc >{\columncolor{ours}}c}
\toprule
Mini-batch size & 30 & 40 & 50 & 60 \\
\midrule
Avg Acc & 74.18  & 74.40 & 74.53  & 74.63  \\
\bottomrule
\end{tabular}
}
\end{table*}

\section{Experiment on DomainNet Benchmark}\label{appendix:domainnet}

To evaluate Potara on the dataset with domain shift, we add experiments on DomainNet, which contains images from six distinct domains: Clipart, Infograph, Painting, Quickdraw, Real, and Sketch. Following prior work \citep{peng2026hilora,wang2024aggregation,yang2023efficient}, we construct a 10-class subset using the most frequent categories. We consider 12 clients, with two clients assigned to each domain. To introduce additional within-domain data heterogeneity, the class proportions for each client are sampled from a Dirichlet distribution with concentration parameter of 0.2. We also include HiLoRA \citep{peng2026hilora} as an additional baseline. For \textsc{Potara}, the FedIT and local models are trained for 200 and 100 rounds, respectively, and are subsequently merged, while all baseline methods are trained for 300 rounds to ensure a comparable overall training budget. Following setting in \citep{peng2026hilora}, the number of rounds per tier for HiLoRA is set to $T_{\text{root}} = 150$, $T_{\text{cluster}}$, and $T_{\text{leaf}}$.
The results show that Potara consistently outperforms baselines, further demonstrating its effectiveness under heterogeneous visual domains.

\begin{table*}[h]
\caption{Performance comparison on DomainNet benchmark.}
\centering
\small
\resizebox{0.9\linewidth}{!}{%
\begin{tabular}{l|ccccccccc >{\columncolor{ours}}c}
\toprule
Method  & FFA-LoRA & FedSA & FedIT & FedDPA  & L2GD & FedMosaic & APFL & PerAda  & HiLoRA &\textsc{Potara}\\
\midrule
Avg Acc &83.15 {\com 0.78} &86.62 {\com 0.03}
&87.32 {\com 0.20} &85.74 {\com 0.02}
&86.70 {\com 0.20} &87.52 {\com 0.09}
&87.57 {\com 0.27} &86.13 {\com 0.02}
&87.71 {\com 0.45} &89.00 {\com 0.11}
\\
\bottomrule
\end{tabular}
}
\end{table*}

\section{Empirical Verification of the LMC Assumption}\label{appendix:mixing_weight}

We note that our setting is fine-tuning from a shared pre-trained model, so both the FedIT and local models can be viewed as relatively small perturbations around an already strong model for related tasks. This suggests that they may remain in a common basin.
To provide empirical evidence for the LMC assumption, we evaluate accuracy along the linear interpolation path.
Specifically, we vary $\lambda_{\text{FedIT}} \in \{0,0.1,\ldots,1\}$ and evaluate client performance under different mixing weights, where $\lambda_{\text{FedIT}}=1$ corresponds to using the FedIT model alone and $\lambda_{\text{FedIT}}=0$ corresponds to the local model. 
We conduct this analysis on commonsense reasoning benchmark with LLaMA where each client has a unique and distinct dataset.
Figure~\ref{fig:lambda_analysis} shows that the performance of nearly all clients exhibits a clear unimodal trend with respect to $\lambda_{\text{FedIT}}$, with intermediate mixing weights outperforming both extreme settings.
This observation suggests that the two models lie in a shared basin and appropriate interpolation can improve performance over either endpoint, motivating the use of model merging.
Moreover, as the FedIT model is trained for more communication rounds, the peak of the unimodal curves shifts noticeably toward larger $\lambda_{\text{FedIT}}$ for nearly all clients.
Consistently, the mixing weights estimated by our method (indicated by red dashed lines) also move to the right as the FedIT round increases.
This behavior indicates that a better-trained FedIT model contributes more reliable general knowledge to the merged model.
This rightward shift further suggests that the Fisher information associated with the FedIT model increases as training proceeds, which in turn assigns a larger weight to the FedIT model during model merging.
Note that the location of the best weight via grid search is sensitive to stochastic training, evaluation noise, the granularity of the search grid, and most crucially, the size of the evaluation set as well as the the degree of task heterogeneity. There is generally no uniquely identifiable or exactly optimal mixing weight. Instead, model interpolation typically yields a neighborhood of weights that achieve similarly strong performance. As shown in Figure~\ref{fig:lambda_analysis}, the interpolation curves are often relatively flat around their maxima, which is consistent with the linear mode connectivity and flat basin behavior commonly observed in model merging. Consequently, the single best weight identified by grid search is itself an uncertain estimate - one sample from a flat high-performing region rather than a unique optimum. Moreover, locating this high-performing region through grid search incurs substantial computational overhead. The goal of Potara is therefore not to exactly reproduce a noisy grid-search maximizer, but to efficiently identify a client-specific weight within the high-performing interpolation neighborhood. The strong empirical performance of our resulting models supports this claim and demonstrates that Potara provides a robust and computationally efficient alternative to exhaustive grid search.

\begin{figure*}[t]
    \centering

    \begin{subfigure}{0.99\linewidth}
        \centering
        \includegraphics[width=\linewidth]{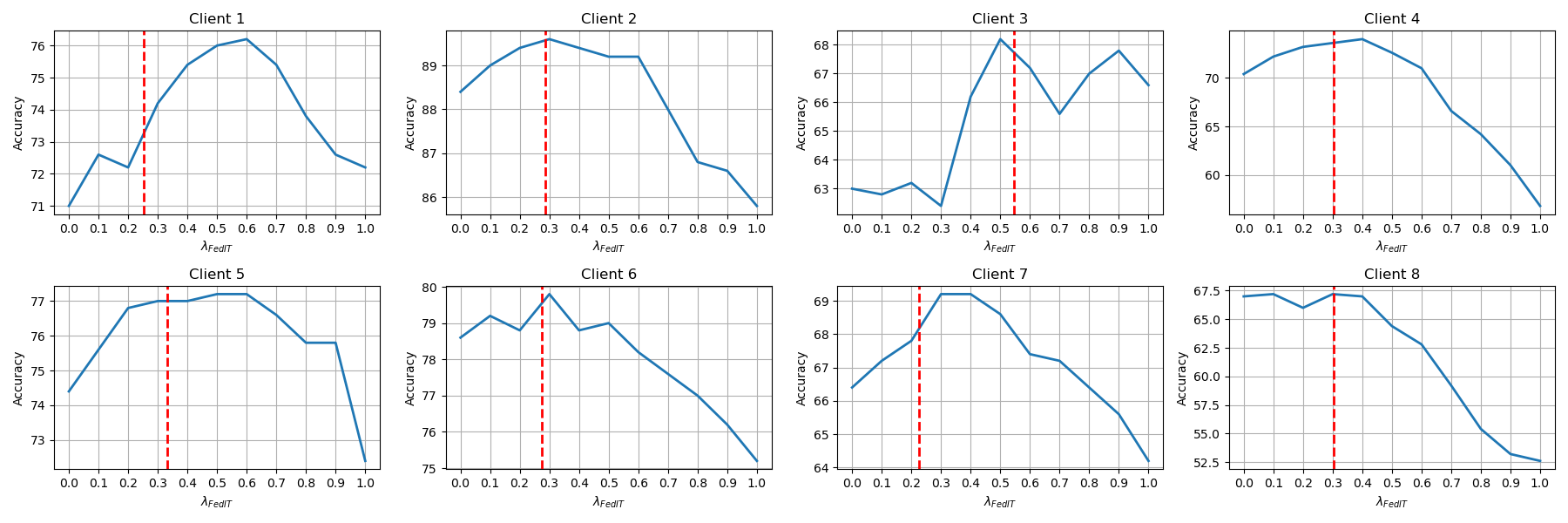}
        \caption{Merging FedIT model at round 50 with local model at round 300 under varying mixing weight.}
        
    \end{subfigure}

    \vspace{2mm}

    \begin{subfigure}{0.99\linewidth}
        \centering
        \includegraphics[width=\linewidth]{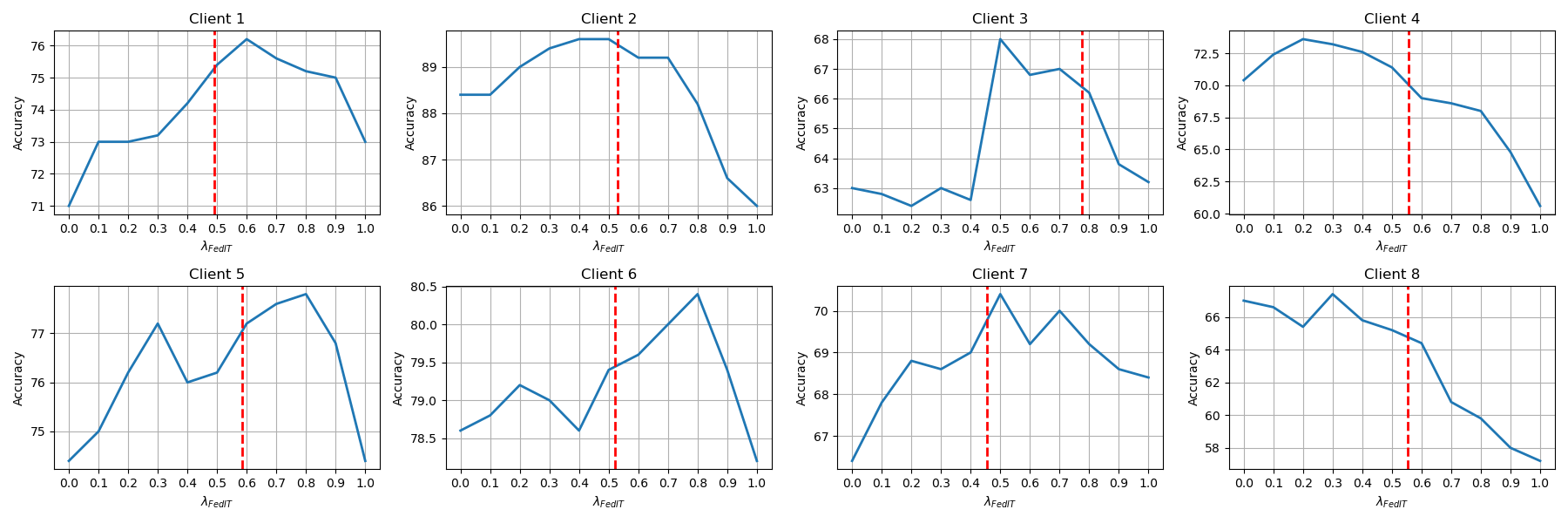}
        \caption{Merging FedIT model at round 150 with local model at round 300 under varying mixing weight.}
    \end{subfigure}

    \vspace{2mm}

    \begin{subfigure}{0.99\linewidth}
        \centering
        \includegraphics[width=\linewidth]{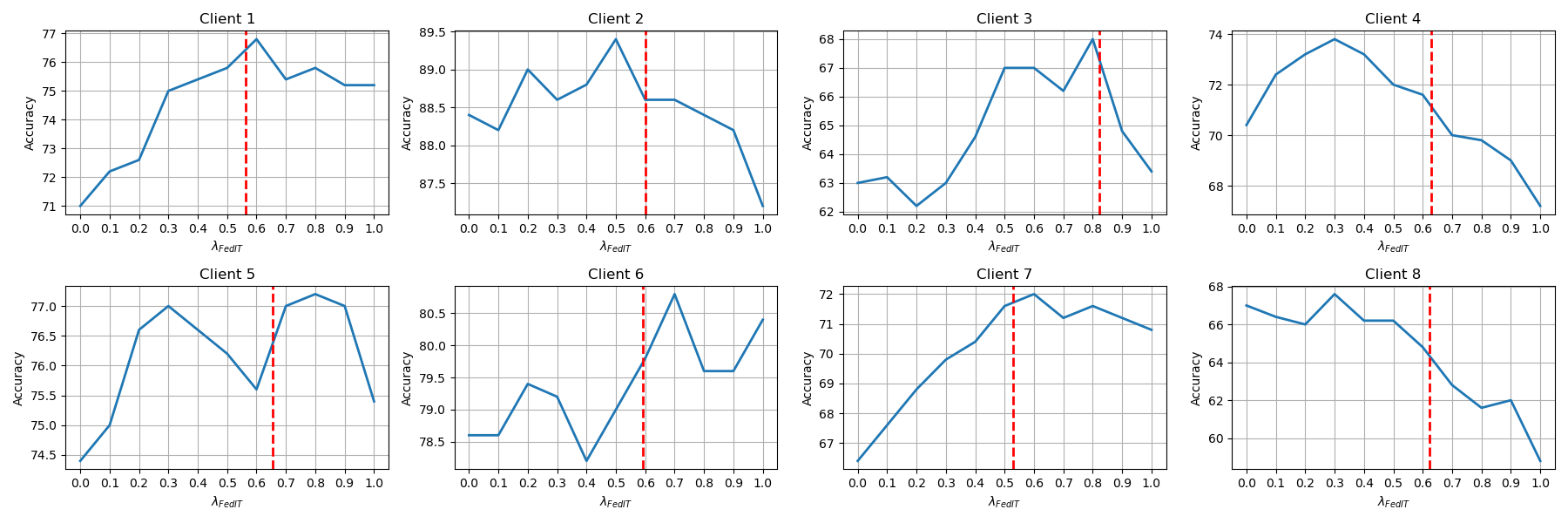}
        \caption{Merging FedIT model at round 300 with local model at round 300 under varying mixing weight.}
    \end{subfigure}

    \caption{Client performance under various mixing weights.
    For each client, we construct a personalized model by merging a FedIT checkpoint at round 50,150, or 300 with a fixed local model at round 300, and sweep the mixing weight \(\lambda_{\text{FedIT}}\) (x-axis). Curves report the resulting client accuracy. 
    The red dashed line marks the mixing weight calculated by our method for each client.
    }
    
    \label{fig:lambda_analysis}
\end{figure*}

\section{Discussion on Gaussian Approximation Assumption}
\label{appendix:Gaussian_assumption}
{\color{black}
The Gaussian approximation is a standard analytical tool in Bayesian deep learning \citep{ritter2018scalable} and model merging \citep{daheimmodel,matena2022merging,wang2025more}. 
Our analysis follows this line of work and uses the Gaussian approximation to characterize the uncertainty of the FedIT and local task vectors. 
The joint Gaussian approximation is motivated by the fact that the FedIT model is obtained from aggregating client updates, which are in turn computed from local client data. As a result, the FedIT model indirectly incorporates information from the local training process, making it reasonable to model a nonzero correlation between the FedIT and local models. From this perspective, the joint Gaussian assumption is introduced to capture this dependence through the cross-covariance term.
}

\section{Implementation Details}\label{appendix:implementation_details}

\subsection{Details on Hyperparameters}

Unless stated otherwise, the hyperparameters used in this work are as follows.

\begin{table}[h]
\centering
\caption{Hyperparameter setting.}
\resizebox{\linewidth}{!}{
\begin{tabular}{l|c|c|c}
\toprule
\textbf{Hyperparameter} 
    & \textbf{ViT-B/16 \& CIFAR-100} 
    & \textbf{LLaMA-3.2-3B-Instruct \& Commonsense Reasoning} 
    & \textbf{Qwen3-4B-Instruct-2507} \&  \textbf{FLAN} \\
\midrule
Batch size                  
    & 10 
    & 4 
    & 4 \\
Learning rate                  
    & 0.001
    & 0.0001
    & 0.0001 \\
LoRA dropout rate           
    & 0.1 
    & 0.05
    & 0.05 \\

Local iteration number 
    & 5 
    & 4
    & 4 \\

Target module               
    & {[}``query'', ``value''{]} 
    & {[}``q\_proj'', ``k\_proj'', ``v\_proj''{]} 
    & {[}``q\_proj'', ``k\_proj'', ``v\_proj''{]} 
    \\

Number of data samples to compute mixing weights
    & 30 
    & 30
    & 30 \\
\bottomrule
\end{tabular}
}
\end{table}

\subsection{Implementation on Baselines}
For baseline implementation, FedMosaic is designed to handle both data heterogeneity and model heterogeneity. Since our experiments focus on the setting where all clients share the same model architecture, we implement FedMosaic with the standard LoRA layers. We implement FedDPA using its FedDPA-T variant, in which the global and local adapters are alternately updated within each communication round. For L2GD, the probability of performing an aggregation step is set to $0.5$. 
Since Fisher Merging requires estimating Fisher information for all model parameters, it is computationally prohibitive for LLMs with over 1B parameters. Therefore, we exclude Fisher Merging from the commonsense reasoning and FLAN benchmarks.

\section{Difference between Fisher Merging and \textsc{Potara}}

Our proposed method \textsc{Potara} differs from Fisher merging in two fundamental aspects.

First, these two methods operate in different parameter spaces, which leads to vastly different computational costs.
Fisher merging is performed in the full parameter space of the model. When the model size is small (e.g., fewer than hundred million parameters), the computation of the Fisher information matrix is manageable and does not introduce significant overhead. However, for LLMs, whose parameter size typically exceeds one billion, computing the Fisher information for all parameters becomes computationally prohibitive.
In contrast, our method performs merging in the task vector space and computes the Fisher information only for the LoRA parameters, whose dimensionality is several orders of magnitude smaller than that of the full model. As a result, our approach incurs substantially lower computational overhead and remains practical for LLMs.

Second, Fisher merging assumes independence between the models being merged and ignores their correlation structure.
Specifically, Fisher merging combines models by weighting them according to their individual Fisher information matrices, without considering correlation between models. 
In contrast, our method explicitly models the correlation between the federated model and the local model, which allows us to derive a more principled and optimal mixing weight.

\end{document}